\documentclass{article}

\PassOptionsToPackage{numbers}{natbib}
\usepackage[preprint]{neurips_2026}

\usepackage[utf8]{inputenc}
\usepackage[T1]{fontenc}
\usepackage{amsmath,amssymb,amsthm}
\newtheorem{definition}{Definition}
\newtheorem{proposition}{Proposition}
\newtheorem{lemma}{Lemma}
\newtheorem{remark}{Remark}
\usepackage{graphicx}
\usepackage{hyperref}
\usepackage{url}
\usepackage{booktabs}
\usepackage{nicefrac}
\usepackage{microtype}
\usepackage{xcolor}
\usepackage{tabularx}
\usepackage{subcaption}

\title{Observable- and Positional-Encoding-Dependent Symmetry Readout from Neural Network Weights}

\author{%
  Naoya Chiba \\
  D3 Center, The Unversity of Osaka\\
  \texttt{chiba@nchiba.net} \\
  \And
  Satoshi Sugiyama \\
  D3 Center, The Unversity of Osaka\\
  \texttt{sugiyama.satoshi.work@gmail.com} \\
  \And
  Yuki Uranishi \\
  D3 Center, The Unversity of Osaka\\
  \texttt{yuki.uranishi.cmc@osaka-u.ac.jp} \\
}

\begin{document}
\maketitle

\begin{abstract}

Post-hoc analysis of trained neural network weights often seeks to recover
geometric structure directly from the parameters. We show that, for
positional-encoding-equipped neural fields, the symmetry visible from weights is
not the true symmetry group itself, but an observable symmetry set determined by
the trained parameters, the positional encoding (PE), and readout observable.
We formulate this dependence through an exact observability hierarchy, $G_{\mathrm{obs}}^{\mathrm{exact}} \subseteq
G_{\mathrm{lift}}^{\mathrm{exact}}(\phi) \cap G_{\mathrm{true}}$, 
where $G_{\mathrm{lift}}^{\mathrm{exact}}(\phi)$ is the set of input
transformations that the PE can exactly lift to the feature space. The hierarchy
implies that even when a target function has a geometric symmetry, that symmetry
may be structurally invisible to weight-level observables if the PE does not
represent the corresponding transformation. We test this prediction using
MLPs trained on two-dimensional signed distance functions with multiple shape
symmetry groups, positional encodings, and Gram-based observables. The results
show a consistent PE-dependent pattern: DyadicAxisPE supports $D_4$-sensitive
readout but structurally suppresses $D_3$ rotations, TriAxisPE yields lower
$D_3$ / $D_6$ readout scores under the tested Gram observables by replacing
coordinate axes with three 120-degree-separated axes,
and 
random Fourier features mainly exhibit a $\pi$-rotation response under these readouts.
These findings show that PE design affects not only approximation behavior
but also which structures are accessible to post-hoc weight-level readouts.
This provides a basis for a principled observable-dependent symmetry readout.

\end{abstract}

\section{Introduction}

Equivariant neural networks \cite{cohen2016gcnn,bronstein2021geometric} have been successful in embedding symmetry groups into the architecture a priori. This paper studies the reverse question: can geometric symmetries be read out post hoc from the weights of a model that was not built to be equivariant?
When tackling this question, this expectation fails for a structural reason: the naive expectation ``the true symmetry group $G_\mathrm{true}$ can be recovered from the weights'' does not hold. What is visible from the weights is the \textbf{observable symmetry set} $G_\mathrm{obs}(\theta; \phi, \Phi)$, and what can be seen is determined not only by the true symmetry but also by the combination of the input representation (PE, $\phi$) and observable ($\Phi$). For example, axis-separable PE can exactly lift up to $D_4$ but does not provide an exact structural readout for $D_3$, whereas generic Random Fourier Features (RFF)  exactly lift only the $\pi$-rotation among the tested rotations, under the general-position assumption. That is, \emph{not all symmetries are equally observable from weight-level structural readouts}, and the exact-liftability bound of the PE constrains the upper limit of the symmetry classes that can be exposed by the tested structural observables.

Crucially, this work does not aim to discover unknown groups freely; instead, it is a post-hoc probing (transform-conditioned readout) that provides candidate transform families ($D_4$, rotation sweep) and evaluates the response of each observable to every transform. Building on this recognition, we propose a framework for observable-dependent symmetry readout centered on the two-factor $(\phi, \Phi)$ structure of $G_\mathrm{obs}$ and establish Proposition~1 (Exact Observability Hierarchy: $G_\mathrm{obs}^{\mathrm{exact}} \subseteq G_\mathrm{lift}^{\mathrm{exact}} \cap G_\mathrm{true}$) as an idealized exact-regime upper bound on post-hoc structural detection. We subsequently conduct empirical tests to test whether the empirical readout profiles are consistent with this hierarchy, specifically examining how the selection of PE and $\Phi$ influences which elements of $G_\mathrm{true}$ are mirrored in $G_\mathrm{obs}$. This is achieved through systematic experiments involving PE-equipped MLPs trained on 2D SDFs, encompassing various shape groups, multiple PEs, and different observables.

\section{Related Work}

This work relates to equivariant architectures, symmetry discovery, and
implicit neural representations, but focuses on what is observable from trained
weights under a given PE and observable readout.

\paragraph{Equivariant architectures and weight-space models.}
Equivariant neural networks embed group structure into the architecture
\cite{cohen2016gcnn,weiler2019steerable,satorras2021egnn} and have been
systematized as geometric deep learning \cite{bronstein2021geometric}. These
methods impose symmetry a priori, whereas our goal is to read out symmetry
post hoc from models that do not explicitly encode the target groups.
DWSNets \cite{navon2023dwsnets} are closest to our setting because they operate
directly on neural network weights and predict properties of implicit neural
representations from weight space. However, they target downstream prediction,
such as classification or regression, rather than which geometric symmetry
groups are structurally observable from trained weights. Our work instead
studies observable-dependent symmetry readout from the trained weights of
PE-equipped MLPs.

\paragraph{Symmetry discovery from data and models.}
LieGG \cite{moskalev2022liegg} extracts Lie algebra generators from gradient
information, LieSD \cite{hu2024liesd} extends related ideas to symmetry
discovery and scoring, and L-conv \cite{dehmamy2021lconv} embeds Lie algebra
structure directly into the network. These methods share a post-hoc or
discovery-oriented motivation with our work, but typically rely on data,
gradients, or model responses. In contrast, our main observable, the
weight-prefix Gram, provides a data-free weight-level readout. They also tend
to treat symmetry as a property of the data or learned function, whereas our
framework emphasizes observable dependence: even for the same trained model,
the visible symmetry can change with the choice of observable $\Phi$ and PE
$\phi$. This leads to the observable symmetry set
$G_\mathrm{obs}(\theta;\phi,\Phi)$ rather than a single recovered group.

\paragraph{Parameter-space and weight-space symmetry.}
Neural network weights contain parameter-space symmetries, such as neuron
permutations and positive rescalings in ReLU networks
\cite{godfrey2022intertwiner,zhao2025paramsym,liu2025parametersymmetry}, that
do not change the represented function. This motivates observables that suppress
irrelevant gauge-like degrees of freedom while retaining geometric information
relevant to the input-space transformations under consideration. The Gram-type
observables used here follow this principle, since $W^\top W$ is invariant to
neuron permutations, though not to all residual scaling degrees of freedom.

\paragraph{INR, positional encodings, and weight representations.}
Implicit neural representations commonly use Fourier features
\cite{tancik2020fourier} or periodic activations \cite{sitzmann2020siren} to
represent high-frequency signals. GRAPE \cite{zhang2025grape} provides a
representation-theoretic classification of positional encodings, identifying
which group actions can be represented by a PE. Our exact-liftability bound
uses this PE-level classification as an upper-bound condition for post-hoc
weight-level readout: a symmetry must first be liftable by the PE before it can
appear as an exact structural symmetry of weight observables. Separately,
Functa \cite{dupont2022functa} and inr2vec \cite{deluigi2023inr2vec} show that
useful information can be extracted from weights or latent representations of
implicit neural fields. We share the premise that neural field parameters encode
structure, but focus specifically on which geometric symmetries are observable
from these parameters and how this depends on the pair $(\phi,\Phi)$.

\section{Theory and Methods}

\subsection{Problem Setting}

This work introduces a framework for observable-dependent symmetry readout, systematically characterizing which geometric symmetries are \textbf{observable} from a trained neural network.
Let $X = \mathbb{R}^2$ be the input space, $H = \mathbb{R}^d$ (where $d$ is the PE output dimension; e.g., $d = 48$ for all structured PEs (DyadicAxisPE: $4K{=}48$ with $K{=}12$; TriAxisPE: $6K{=}48$ with $K{=}8$; RFF: $2n{=}48$ with $n{=}24$)) be the output space (feature space) of the positional encoding (PE), $\phi : X \to H$ be the PE, and $F_\theta : H \to \mathbb{R}$ be the subsequent neural network. The overall model is given by
\[
f_\theta(x) = F_\theta(\phi(x))
\]
where $\theta$ denotes all the trained parameters.

The geometric symmetries of interest are described by a transformation group $G$ that acts on the input space $X$. Each $g \in G$ acts on a point $x \in X$, and on a function $f$ via
\[
(g \cdot f)(x) = f(g^{-1} x).
\]
An object is said to be symmetric with respect to a group $G_\mathrm{true}$ when the corresponding true function $f_*$ satisfies
\[
f_*(g^{-1}\, x) = f_*(x) \quad \text{for all } g \in G_\mathrm{true}.
\]

Naively, one would like to directly recover $G_\mathrm{true}$ from the trained weights $\theta$. In practice, however, the weights themselves do not directly retain geometric groups in the input space; symmetry manifests indirectly through positional encoding, the network's internal representations, and the choice of observable. Therefore, this work first provides a theoretical formulation of what symmetry is visible from the weights and then organizes what is recoverable and what is in principle difficult to observe.

\subsection{Positional Encoding as a Representation}

Positional encoding (PE) is a mapping that projects low-dimensional input coordinates into a higher-dimensional feature space using trigonometric functions or similar bases, and is widely used in implicit neural representations (INRs) such as Neural Radiance Fields (NeRF) \cite{mildenhall2020nerf,tancik2020fourier,sitzmann2020siren}. Plain MLPs tend to preferentially learn low-frequency functions (spectral bias \cite{rahaman2019spectral}), and PE mitigates this bias by lifting inputs into a higher-dimensional space, enabling the learning of functions with high-frequency content.

In this work, we view PE $\phi$ not merely as preprocessing, but as a feature map that lifts geometric transformations to the feature space \cite{zhang2025grape}. When a geometric transformation $g$ (rotation, reflection, translation, etc.) is applied to an input $x$, passing $gx$ through the PE yields $\phi(gx)$. A natural question is whether $\phi(gx)$ can be computed directly from $\phi(x)$, that is, whether the geometric transformation can be represented as an operation in the feature space.

We focus on the case where $\phi(gx)$ is expressible as a linear transformation of $\phi(x)$.
That is, there exists a linear map $\rho(g)$\footnote{In general, PEs with bias terms may induce affine actions, but for all PEs considered in this paper (DyadicAxisPE, TriAxisPE, and RFF), $\rho(g)$ is realized as a linear map.} on $H$ such that
\[
\phi(g\, x) = \rho(g)\, \phi(x),
\]
$\phi$ is said to be $G$-equivariant with respect to $\rho$. Here, $\rho$ is the group representation that lifts the geometric action on the input space to the feature space. The linearity of this lift is essential because each MLP layer is a composition of a linear transformation (weight matrix) and nonlinear activation. A linear lift on the PE allows the effect of geometric transformations to be tracked through algebraic operations on weight matrices, providing a foundation for symmetry readout from weights.

From this perspective, the positional encoding determines which input transformations can be tracked exactly as linear feature-space actions.
That is, if $\phi$ is $G$-equivariant for a given transformation group, the group action can be explicitly tracked in feature space. Conversely, for transformations for which no such lift exists, even if the symmetry is natural in the input space, it does not appear as an exact representation within the model.

Therefore, the problem of reading out symmetry from trained weights depends first on how well the positional encoding retains the group actions. This is the theoretical starting point of this work.

We define the PEs used in this work and consolidate their exact liftability in Lemma~\ref{lem:exact-lift}.

\paragraph{DyadicAxisPE.}
An axis-separated PE with dyadic frequen, $\rho(g)$ is realized as a linear map. cies \cite{tancik2020fourier}, with an output dimension $4K$ for an octave count $K$ is given by
\[
\phi_\mathrm{Dyadic}(x) = \bigl[\{\sin(2^k \pi x_1),\, \cos(2^k \pi x_1),\, \sin(2^k \pi x_2),\, \cos(2^k \pi x_2)\}_{k=0}^{K-1}\bigr].
\]

\paragraph{Random Fourier Features (RFF).}
An isotropic PE using randomly sampled frequency vectors $\omega_i \in \mathbb{R}^2$ and biases $b_i \in \mathbb{R}$ ($i = 1, \dots, n$) \cite{rahimi2007random,tancik2020fourier} with an output dimension $2n$ is defined as
\[
\phi_\mathrm{RFF}(x) = \bigl[\sin(\omega_1^\top x + b_1),\, \cos(\omega_1^\top x + b_1),\, \dots,\, \sin(\omega_n^\top x + b_n),\, \cos(\omega_n^\top x + b_n)\bigr].
\]

\paragraph{Tri-Axis Dyadic PE (TriAxisPE).}
Extension of DyadicAxisPE by replacing the two coordinate axes with three axes separated by $2\pi/3$. Let $v_0 = (1,0)^\top$, $v_1 = R_{2\pi/3}\,v_0$, $v_2 = R_{4\pi/3}\,v_0$. The encoding is:
\[
\phi_\mathrm{tri}(x) = \bigl[\{\sin(2^k\pi\, v_j^\top x),\, \cos(2^k\pi\, v_j^\top x)\}_{j=0,\ldots,2;\,k=0,\ldots,K-1}\bigr],
\]
with an output dimension $6K$.

\begin{lemma}[Exact rotation/reflection liftability of structured PEs]
\label{lem:exact-lift}
For the three PEs defined above (
assuming the RFF frequency set $\{\omega_i\}$ is in general position, i.e., it is not closed under the tested rotations or reflections except for sign reversal):
\begin{enumerate}
\item[(i)] $\phi_\mathrm{Dyadic}$ exactly lifts $D_4$ (generated by $\pi/2$ rotation and axis reflections); transformations outside $D_4$ (in particular $\pi/3$ and $2\pi/3$ rotations) are not exactly lifted.
\item[(ii)] $\phi_\mathrm{tri}$ exactly lifts $D_6$ (containing $D_3$); transformations outside $D_6$ (in particular $\pi/2$ rotation) are not exactly lifted.
\item[(iii)] $\phi_\mathrm{RFF}$ exactly lifts $Z_2$ ($\pi$ rotation); transformations outside $Z_2$ (rotations at other angles and reflections) are not exactly lifted.
\end{enumerate}
\end{lemma}

\begin{remark}
DyadicAxisPE and TriAxisPE have complementary exact liftability bounds ($D_4$ and $D_6 \supset D_3$, respectively). The non-liftability of $D_3$ under DyadicAxisPE is specific to its axis-separable structure and is not an inherent limitation of dyadic-frequency PEs. TriAxisPE recovers exact liftability for $D_3$ and $D_6$ by the minimal change of replacing coordinate axes with three $2\pi/3$-separated axes, while preserving the dyadic frequency and sine/cosine pair structure.
\end{remark}

\noindent
The proof is provided in Appendix~\ref{app:proofs}.

\subsection{Functional Symmetry and Observable Symmetry}

The true symmetry of an object is inherently defined as the invariance of the function $f_*$: $f_*(gx) = f_*(x)$. The most direct way to verify this with a trained model $f_\theta$ is dense forward computation, which evaluates $f_\theta(gx) \approx f_\theta(x)$ at many input points. However, this is an input--output-level verification of how well $f_\theta$ approximates $f_*$ and does not reveal how symmetry is encoded inside the model. The goal of this work is not such input--output verification, but rather reading out symmetry from the weights and internal representations of the trained model (Figure~\ref{fig:overview}).

Accordingly, we write the observable computed from the trained model as $\Phi(\theta)$. $\Phi$ may be a quantity in the weight space or may use internal representations or limited activation information as needed. The key point is that $\Phi$ serves as a window for summarizing the internal structure of the model and measuring its stability under geometric transformations.

For a transformation $g \in G$, we define the corresponding action $T_g$ on the model’s side. The $g$-transformed model $T_g\,\theta$ is the parameter representing the shape $g \cdot S$ obtained by applying $g$ to the shape $S$ that $\theta$ represents, i.e.,
\[
f_{T_g\,\theta}(x) = f_\theta(g^{-1}x).
\]
When $g \in G_\mathrm{lift}^{\mathrm{exact}}$, we have $\phi(g^{-1}x) = \rho(g^{-1})\,\phi(x)$, so
\[
f_{T_g\,\theta}(x) = F_\theta\!\bigl(\rho(g^{-1})\,\phi(x)\bigr).
\]
This definition makes $T$ a genuine left action on the parameter space ($T_{gh} = T_g T_h$), consistent with the definition of $\rho$.

Since $\rho$ acts only on PE outputs, $T_g\theta$ is realized simply by modifying the first-layer weight (which receives the PE output) as $W_0 \to W_0\,\rho(g^{-1})$, leaving subsequent layers and all biases unchanged; this realization holds exactly regardless of activation functions. Consequently, any quantity built from $W_0$ inherits the $\rho$-action. As the simplest example, taking the Gram matrix $G_0 = W_0^\top W_0$ at the PE output as the observable gives
\[
G_0(T_g\,\theta) = \rho(g^{-1})^\top\, G_0\, \rho(g^{-1})
\]
\footnote{In experiments, we use the homogeneous extension $\tilde{G}_0 = \bigl[\begin{smallmatrix} G_0 & W_0^\top b_0 \\ b_0^\top W_0 & \|b_0\|^2 \end{smallmatrix}\bigr]$ with the corresponding $\tilde{\Pi}_g = \bigl[\begin{smallmatrix} \rho(g^{-1}) & 0 \\ 0 & 1 \end{smallmatrix}\bigr]$ to account for bias terms (Appendix~\ref{app:grid}).}. This is precisely the prefix Gram $L_0$ used in our main experiments. Deeper prefixes ($P_l = W_l \cdots W_0$ with prefix Gram $G_l = P_l^\top P_l$) and the effective weight $W_\mathrm{eff} = W_L \cdots W_0$ (the end-to-end linear map when activations are ignored, a quantity extensively studied in deep linear network analysis \cite{saxe2014exact,arora2019implicit}) satisfy $P_l \to P_l\,\rho(g^{-1})$ and $W_\mathrm{eff} \to W_\mathrm{eff}\,\rho(g^{-1})$, inducing the same $\rho$-sandwich structure as the shallow prefix. A similar structure is induced via $\rho$ for general observable quantities.
Since $G$ is closed under inversion, the detection set $G_\mathrm{obs} = \{g \mid \Phi(T_g\theta) = \Phi(\theta)\}$ is invariant under $g \leftrightarrow g^{-1}$.

In this framework, the group action is handled in the feature space via $\rho$, which differs from dense forward ($f_\theta(g^{-1}x) \approx f_\theta(x)$) that acts on the input space in the space on which the group action is applied. The universality of MLPs allows training to realize function-level symmetry (\textbf{functional symmetry}) independently of PE liftability, whereas the invariance of observables via $\rho$ (\textbf{structural symmetry}) is constrained by the PE's algebraic structure. The subsequent $G_{\mathrm{obs}}(\theta;\phi,\Phi,\mathcal{T})=\{g\in\mathcal{T}: \Phi(T_g\theta)\approx\Phi(\theta)\}$ and hierarchy are described at the structural level. The gap between the two notions and the dual role of the PE in training and observation are detailed in Appendix~\ref{app:dense_forward}.

\begin{definition}[Observable Symmetry Set]
\label{def:obs-sym-set}
The observable symmetry set with respect to trained parameters $\theta$, PE $\phi$, and observable $\Phi$ is defined as
\[
G_\mathrm{obs}(\theta; \phi, \Phi) = \{ g \in G \mid \Phi(T_g \theta) \approx \Phi(\theta) \}.
\]
Here, $\approx$ denotes approximate equality that accounts for numerical error, approximate representations, and training error.
\end{definition}

The above definition serves as a generic placeholder for observable symmetry sets. In the following, we first define the exact version $G_\mathrm{obs}^{\mathrm{exact}}$ restricted to $g \in G_\mathrm{lift}^{\mathrm{exact}}$, and then introduce the operational extension $D_\mathrm{op}$ that allows approximate lifts for $g \notin G_\mathrm{lift}^{\mathrm{exact}}$. The operational quantification of the approximate equality $\approx$ is formalized as the symmetry score $S(g;\,\theta,\Phi)$ in \S3.5.

The key point of this definition is to capture symmetry not as ``an absolute property that the model possesses,'' but as an invariant structure that appears stable through a given observable. Therefore, in this work, symmetry is generally not a property of $\theta$ alone, but a quantity that depends on the triple $(\theta, \phi, \Phi)$.

\begin{figure}[htbp]
\centering
\includegraphics[width=\linewidth]{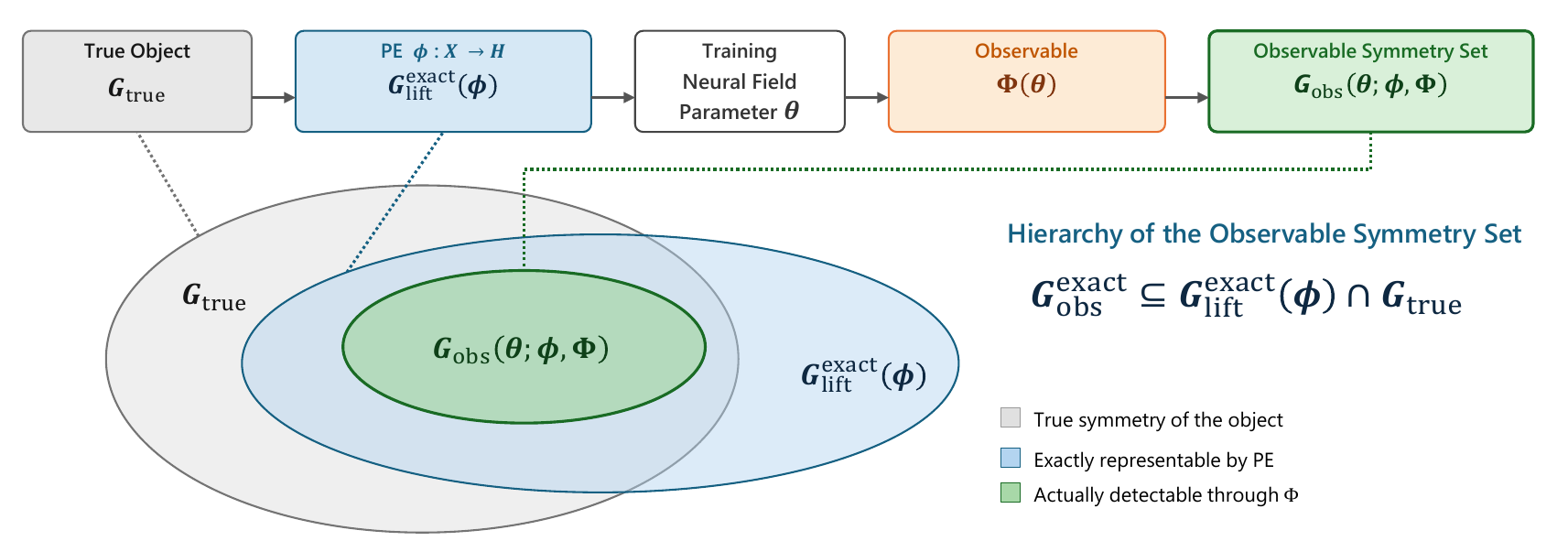}
\caption{Overview of the observable-symmetry readout framework. Weight-level
readout yields an observable symmetry set
$G_\mathrm{obs}(\theta;\phi,\Phi)$, not $G_\mathrm{true}$ itself. The Venn
diagram shows the exact-regime hierarchy
$G_\mathrm{obs}^{\mathrm{exact}} \subseteq
G_\mathrm{lift}^{\mathrm{exact}}(\phi) \cap G_\mathrm{true}$; the dashed region
indicates the operational regime based on approximate lifts.}
\label{fig:overview}
\end{figure}

\subsection{Hierarchy of Observable Symmetry}

Observable symmetry has three levels (Figure~\ref{fig:overview}). (1) The true symmetry group $G_\mathrm{true}$ of the object; (2) The PE-liftable group $G_\mathrm{lift}^{\mathrm{exact}}(\phi) = \{ g \in G \mid \phi(gx) = \rho(g)\,\phi(x) \}$ is determined solely by PE's algebraic properties of PE. (3) The observable symmetry set $G_\mathrm{obs}^{\mathrm{exact}}(\theta; \phi, \Phi) \subseteq G_\mathrm{lift}^{\mathrm{exact}} \cap G_\mathrm{true}$, which is the symmetry that is visible through the chosen observable $\Phi$. This nested structure forms the theoretical backbone of the study.

\begin{definition}[Symmetry-sufficient observable]
\label{def:sym-sufficient}
An observable $\Phi$ is said to be \textbf{symmetry-sufficient} if for any $g \in G_\mathrm{lift}^{\mathrm{exact}}$,
\[
\Phi(T_g\, \theta) = \Phi(\theta) \;\Longrightarrow\; f_\theta(g^{-1}\, x) = f_\theta(x) \quad \text{for all } x \in X.
\]
Invariance at the observable level implies invariance at the functional level.
\end{definition}

\begin{proposition}[Exact Observability Hierarchy]
\label{prop:hierarchy}
If $f_\theta$ has sufficiently converged to $f_*$ ($\sup_{x \in X} |f_\theta(x) - f_*(x)| \leq \delta,\; \delta \to 0$), and $\Phi$ is compatible with the induced $T_g$ action and symmetry-sufficient (Definition~\ref{def:sym-sufficient}), then
\[
G_\mathrm{obs}^{\mathrm{exact}}(\theta;\, \phi, \Phi) \;=\; \{ g \in G_\mathrm{lift}^{\mathrm{exact}} \mid \Phi(T_g\,\theta) = \Phi(\theta) \}
\]
satisfies $G_\mathrm{obs}^{\mathrm{exact}}(\theta; \phi, \Phi) \subseteq G_\mathrm{lift}^{\mathrm{exact}}(\phi) \cap G_\mathrm{true}$.
\end{proposition}

\noindent
The proof is provided in Appendix~\ref{app:proofs}. The experiments test the consistency of this hierarchy.

This hierarchy cleanly separates three cases: symmetry does not exist, symmetry exists in the representation but is not observed, and symmetry is not lifted into the representation. Training convergence also affects $G_\mathrm{obs}$ (Appendix~\ref{app:dynamics}); the main results isolate the PE and $\Phi$ effects while holding this factor fixed.

\subsection{Quantification of Symmetry: Exact Identification and Operational Detection}

The \textbf{symmetry score} of transformation $g$ with respect to observable $\Phi$ is
\[
S(g;\, \theta, \Phi) = \frac{\| \Phi(T_g\, \theta) - \Phi(\theta) \|_F}{\| \Phi(\theta) \|_F}
\]
where $\|\cdot\|_F$ denotes the Frobenius norm. $S = 0$ signifies a perfect symmetry. In the experiments, instead of testing $S = 0$ directly, we used the thresholded set $G_\mathrm{obs}^{\varepsilon} = \{ g \in G_\mathrm{lift}^{\mathrm{exact}} \mid S(g;\,\theta,\Phi) < \varepsilon \}$ (for $\varepsilon > 0$, $G_\mathrm{obs}^{\varepsilon} \supseteq G_\mathrm{obs}^{\mathrm{exact}}$).

\paragraph{PE-level Procrustes residual.}
Whether PE $\phi$ admits an exact lift for transformation $g$ is determined by the orthogonal Procrustes residual \cite{schonemann1966generalized}
\[
r_P(g) = \min_{Q \in O(d)} \frac{\|\phi(gX) - Q\,\phi(X)\|_F}{\|\phi(gX)\|_F}.
\]
For all PEs in this work, exact lifts were realized as orthogonal transformations (permutation matrices or block rotation matrices), making orthogonal Procrustes the appropriate estimator. $r_P \approx 0$ (machine precision $\sim 10^{-14}$) confirms the existence of an exact lift, $r_P > 0$ confirms its absence.

For transforms $g$ outside $G_\mathrm{lift}^{\mathrm{exact}}$, the exact lift $\rho(g)$ does not exist, so $\Phi(T_g\theta) \approx \Phi(\theta)$ in Definition~\ref{def:obs-sym-set} cannot be evaluated as is. We define the best orthogonal approximation as the solution to the above Procrustes problem as follows:
\[
\hat{\rho}(g) = \operatorname*{argmin}_{Q \in O(d)} \|\phi(gX) - Q\,\phi(X)\|_F
\]
(numerical procedure in Appendix~\ref{app:rff_scaling}). Since $\hat\rho(g) = \rho(g)$ when $g \in G_\mathrm{lift}^{\mathrm{exact}}$, replacing $\rho$ by $\hat\rho$ defines an operational approximation outside the exact regime. The operational score $S_\mathrm{op}(g;\,\theta, \Phi)$ is defined as the symmetry score with $\rho$ replaced by $\hat\rho$ in the computation of $\Phi(T_g\,\theta)$, and the operational detection set is $D_\mathrm{op}(\theta;\, \Phi,\, \varepsilon) = \{ g \in \mathcal{T} \mid S_\mathrm{op}(g;\, \theta, \Phi) < \varepsilon \}$ where $\mathcal{T}$ is the sampled transform family. The theory section uses the exact regime ($G_\mathrm{obs}^{\mathrm{exact}}$), whereas the experiments use the operational regime ($D_\mathrm{op}$). Because the approximation error in the learned weights makes $S(g) > 0$ at the observable level even for exact-lift transforms (e.g., $S \approx 0.2$ for $D_4$ transforms under DyadicAxisPE $+$ prefix Gram $L_0$), $\varepsilon$ is used in experiments for basic validation (absence of false positives), whereas the primary analysis relies on relative score profiles rather than hard-threshold detection.

From the above organization, the experiments address two primary questions: (i) how the PE constrains the upper bound of observable symmetry, and (ii) within weight-only observables, how dominant is the choice of PE in determining the readout sensitivity.

\section{Experiments and Results}

\subsection{Experimental Setup}

We extracted geometric symmetries from the post-training weights and internal representations of MLPs (depth 5, width 128, ReLU) trained on 2D signed distance functions (SDFs). Training protocol: 10,000 points sampled from $[-1,1]^2$ (50\% uniform, 50\% boundary-enriched), batch size 256, MSE loss, Adam \cite{kingma2015adam} ($\mathrm{lr}=10^{-3}$), CosineAnnealingLR \cite{loshchilov2017sgdr}, 2,000 epochs, rotation sweep at 72 angles ($5^{\circ}$ intervals) with 512 sample points for Procrustes estimation, 5 seeds per condition. The main text demonstrates representative cases using six shapes; the Appendices assess robustness and variation across all 16 shapes, the full PE $\times$ observable grid, noise sensitivity, learning dynamics, and translational symmetry.

\paragraph{Target shapes.}
The main experiments used six shapes (Figure~\ref{fig:shape_gallery}): one representative per discrete symmetry group ($O(2)$: circle, $D_6$: hexagon, $D_4$: square, $D_3$: equilateral\_triangle, $D_2$: ellipse) plus a trivial-group baseline random\_noise, covering the partial orders $O(2) \supset D_4 \supset D_2$ and $O(2) \supset D_6 \supset D_3$. Each condition was independently trained using five seeds. Translation-symmetric shapes ($p1m$; stripe) and within-group second representatives, totaling 10 additional shapes, are discussed in Appendices ~\ref{app:translation} and ~\ref{app:grid} (Figure~\ref{fig:shape_gallery_appendix}).

\begin{figure}[t]
\centering
\includegraphics[width=\linewidth]{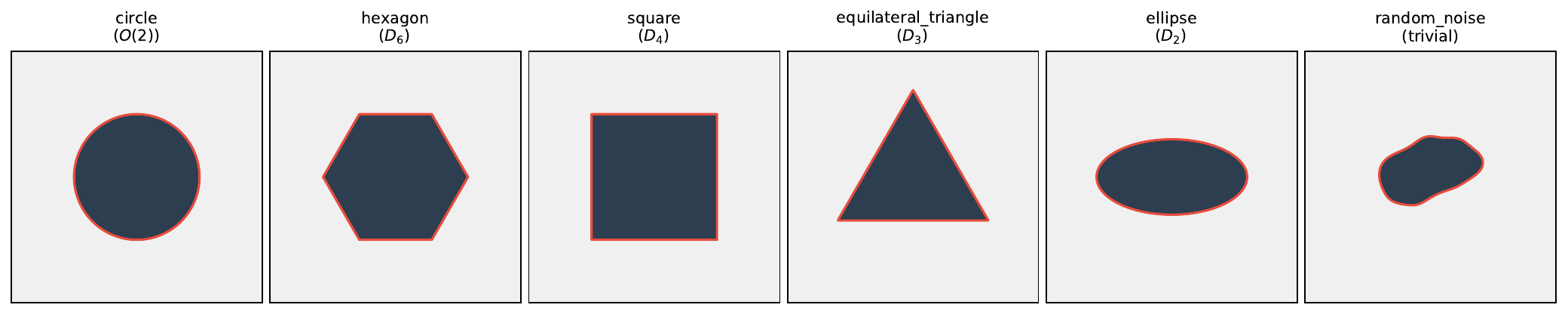}
\caption{Zero level sets of the SDFs for the 6 shapes used in the main text (red: boundary; dark: interior $f_* < 0$).}
\label{fig:shape_gallery}
\end{figure}

\paragraph{Positional encodings.}
The main comparisons use the three structured PEs: \textbf{DyadicAxisPE} ($K=12$, $4K=48$ dims; exact-liftability bound $D_4$), \textbf{TriAxisPE} ($K=8$, $6K=48$ dims; lifts $D_6$/$D_3$ but not $D_4$), and \textbf{RFF} ($n=24$, $2n=48$ dims; exact-liftability bound $Z_2$). All three were trained independently under the same MLP architecture and training setup and shared the same 48-dimensional representation, enabling a dimension-controlled comparison of the encoding structure effects. The RFF choice $n=24$ lies safely within the reconstruction saturation regime ($n \ge 4$).

\paragraph{Symmetry score.}
Using the symmetry score $S(g;\,\theta, \Phi)$ and threshold $\varepsilon$. We report hard detection, \(S<\varepsilon\), only for false-positive analysis; the main figures analyze raw and relative score profiles. Varying $\varepsilon$ from $0.02$ to $0.10$ preserves the primary detection/non-detection patterns (sensitivity analysis in Appendix~\ref{app:dynamics}). At the PE level, Procrustes residuals reaching machine precision ($\sim 10^{-14}$) confirm the existence of exact lifts, whereas the observable-level $S(g)$ attained depends on $\Phi$. In what follows, $S(g)$ for the case where $\Phi$ is a Gram matrix is abbreviated as the Gram distance $\Delta$, with the observable specified (e.g., prefix Gram $L_0$ $\Delta$, $W_\mathrm{eff}$ Gram $\Delta$).

We note that no false positives ($D_\mathrm{op}$ exceeding $G_\mathrm{true}$) were observed under the conditions of this work; this is empirically verified in Appendix~\ref{app:dynamics}. Below, we verify the exact liftability bound  as our primary empirical result.

\subsection{PE Exact Liftability and Its Relation to Weight-Level Readout Responses}

The true symmetry $G_\mathrm{true}$ does not appear directly in the weights; it is first restricted to the PE-representable $G_\mathrm{lift}^{\mathrm{exact}}$, and then the observable-dependent $G_\mathrm{obs}$ is obtained. We fix the observable to prefix Gram $L_0$ and examine how the PE's exact liftability constrains this upper bound. prefix Gram $L_0$ is the shallowest layer of the weight-prefix Gram $G_l = P_l^\top P_l$ ($P_l = W_l \cdots W_0$); it contains only the PE's linear transformation and therefore most directly preserves the symmetric structure of PE space (chosen for mechanistic clarity; systematic comparison with $L_0$--$L_4$ and $W_\mathrm{eff}$ Gram, and with other observable choices, is deferred to Appendix~\ref{app:prefix_layers},~\ref{app:grid}).

\subsubsection{PE-Level Exact Lift and Its Consequences}

\begin{figure}[t]
\centering
\includegraphics[width=\linewidth]{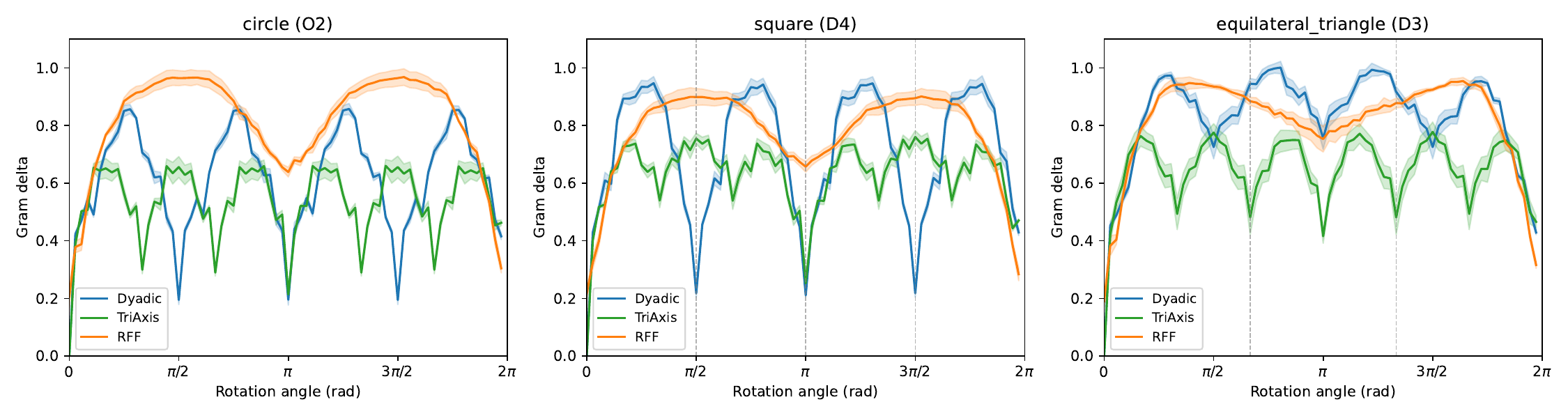}
\caption{Prefix Gram $L_0$ score $\Delta$ over rotation angles for 3 shapes $\times$ 3 PEs. DyadicAxisPE shows $D_4$-aligned dips, TriAxis gives lower scores on the $D_3$ shape, and RFF mainly shows a $\pi$ dip. Shading: $\pm1$ std over 5 seeds. See Figure~\ref{fig:observability_map_main} and Appendix~\ref{app:prefix_layers} for extended comparisons.}
\label{fig:rotation_curves}
\end{figure}

Table~\ref{tab:procrustes} reports the Procrustes residuals $r_P(g)$ for the three PEs. Within the exact-liftable ranges of Lemma~\ref{lem:exact-lift} (DyadicAxisPE's $D_4$, TriAxisPE's $D_6 \supset D_3$, RFF's $Z_2$), $r_P$ reaches near-zero numerical residuals provide evidence for exact lifts, whereas outside these ranges, $r_P \sim 10^{-1}$--$10^{0}$, orders of magnitude larger, consistent with the theoretical liftability bounds. The translation residuals are reported in Appendix~\ref{app:translation}.

\begin{table}[t]
\centering
\caption{PE-level Procrustes residuals $r_P(g)$. Values are averaged over five input samples for deterministic PEs and over 25 trials for RFF. Near-zero residuals indicate exact lifts; large residuals indicate non-liftability.}
\label{tab:procrustes}
\begin{tabular}{lccc}
\hline
Transform & DyadicAxisPE & TriAxis & RFF \\
\hline
$\pi/2$ rotation & $< 10^{-14}$ & $1.2 \times 10^{0}$ & $2.0 \times 10^{-1}$ \\
$\pi/3$ rotation & $1.2 \times 10^{0}$ & $4.6 \times 10^{-6}$ & $1.9 \times 10^{-1}$ \\
$2\pi/3$ rotation & $1.2 \times 10^{0}$ & $4.6 \times 10^{-6}$ & $1.8 \times 10^{-1}$ \\
$\pi$ rotation & $< 10^{-14}$ & $< 10^{-14}$ & $7.3 \times 10^{-8}$ \\
$x$-reflection & $< 10^{-14}$ & $< 10^{-14}$ & $1.4 \times 10^{-1}$ \\
\hline
\end{tabular}
\end{table}

Figure~\ref{fig:rotation_curves} shows the angular response of prefix Gram $L_0$ $\Delta$ (distinct from the Procrustes residual, it measures invariance of the \emph{trained} shallowest-layer Gram). DyadicAxisPE and TriAxisPE exhibit complementary dominance on $D_4$ shapes and $D_3$ shapes, respectively, while RFF shows a dip at $\pi$ but lacks periodic structure. The pattern is consistent with our basic position that ``the PE's exact-liftable group type fixes the visible structure, and shape modulates its strength.'' A six-group systematic comparison, including shapes such as hexagons ($D_6$), where both PEs compete via the common subgroup $D_4 \cap D_6 = D_2$, is shown in Figure~\ref{fig:observability_map_main}. Under the tested prefix Gram $L_0$, no systematic valley formation was observed outside the exact liftability bound.

\subsubsection{\texorpdfstring{Suppression of $D_3$ Rotation Responses Under DyadicAxisPE}{Suppression of D3 Rotation Responses Under DyadicAxisPE}}

\begin{figure}[t]
\centering
\begin{subfigure}[t]{0.49\linewidth}
\centering
\includegraphics[width=\linewidth]{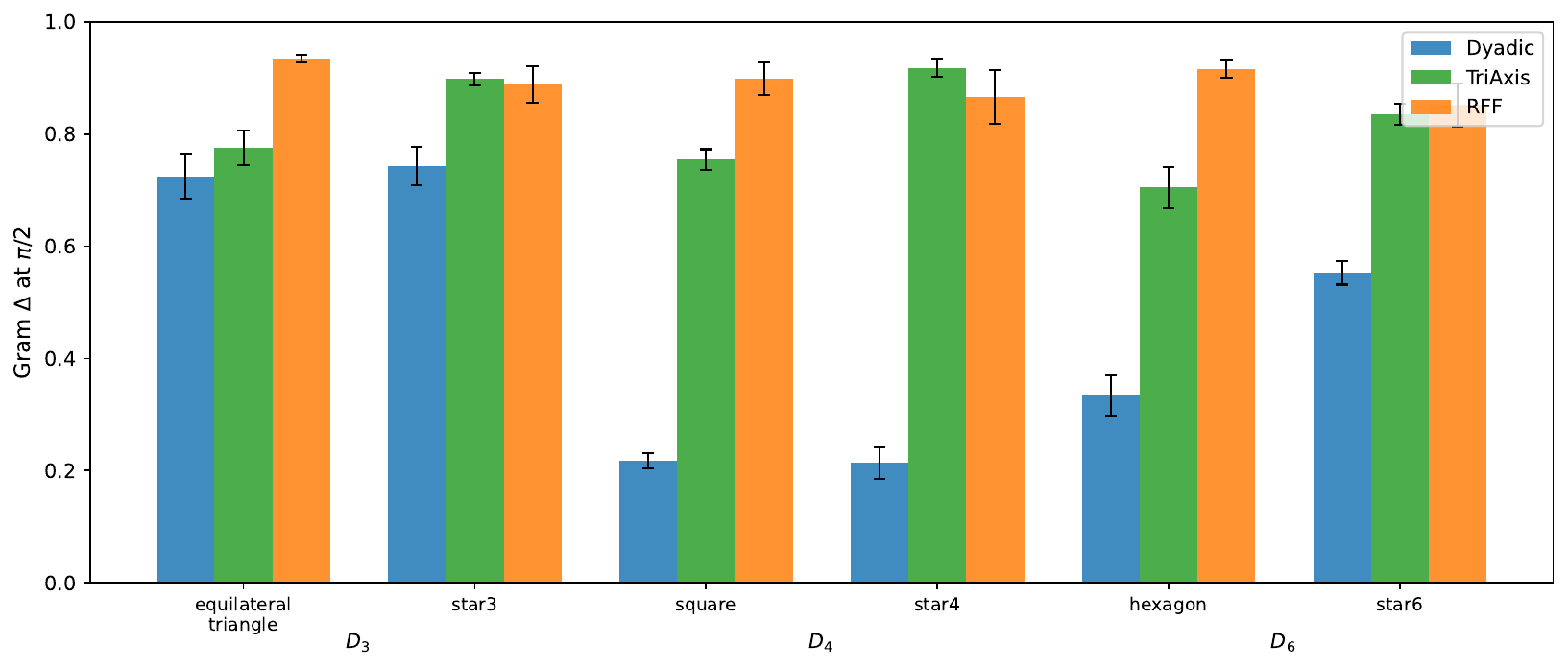}
\caption{$\pi/2$ rotation (relative-response anchor)}
\label{fig:d4_nondetect}
\end{subfigure}
\hfill
\begin{subfigure}[t]{0.49\linewidth}
\centering
\includegraphics[width=\linewidth]{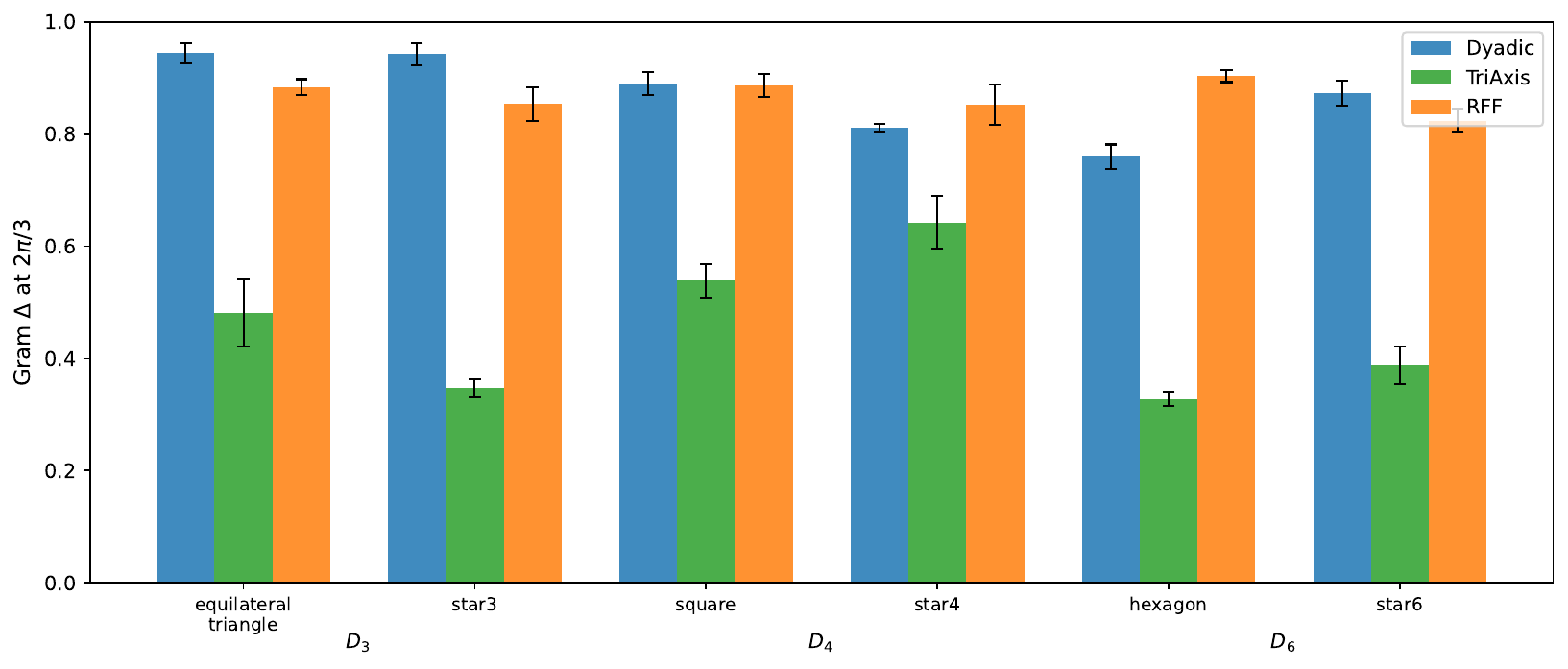}
\caption{$2\pi/3$ rotation (structural limitation)}
\label{fig:d3_nondetect}
\end{subfigure}
\caption{Prefix Gram $L_0$ score $\Delta$ for 3 PEs $\times$ 6 shapes. DyadicAxisPE gives the lowest $\pi/2$ scores on $D_4$ shapes, whereas TriAxis gives the lowest $2\pi/3$ scores on $D_3$ shapes, though still above $\varepsilon=0.05$. Results are mean $\pm1\sigma$ over 5 seeds.}
\label{fig:nondetect_pair}
\end{figure}

Below, we examine a positive example (DyadicAxisPE's $\pi/2$) and a structural limitation ($D_3$ non-detection) in contrast. Since this subsection addresses group-level detection for $D_3$/$D_4$/$D_6$, Figure~\ref{fig:nondetect_pair} uses two shapes per group (a main-text representative plus an Appendix shape; $D_3$: equilateral\_triangle + star3, $D_4$: square + star4, $D_6$: hexagon + star6) to confirm within-group reproducibility. Among the six main-text shapes, $O(2)$ (circle), $D_2$ (ellipse), and the trivial group (random\_noise) are outside the scope of this subsection and are deferred to \S4.2.1 and the Appendix. Both panels share the same y-axis for a direct comparison.

Positive example (Figure~\ref{fig:d4_nondetect}): only $D_4$ shapes drop substantially under the DyadicAxisPE, consistent with $\pi/2 \in D_4$; TriAxis and RFF remain high across all shapes (neither lifts $\pi/2$). Structural limitation (Figure~\ref{fig:d3_nondetect}): under DyadicAxisPE, $D_3$ shapes score systematically higher than $D_4$/$D_6$ shapes, reflecting that $\pi/3$ and $2\pi/3$ are contained in neither DyadicAxisPE's $D_4$ bound nor RFF's $Z_2$ bound. In contrast, the TriAxis panel shows a substantial drop in $D_3$ shapes (equilateral\_triangle: $0.48$, star3: $0.35$), yielding a $D_3$-sensitive relative response under the tested Gram readout via a PE that exactly lifts $D_3$. An analogous complementary pattern holds for $D_6$ shapes (hexagon, star6): at $\pi/3$, TriAxisPE drops to $0.33$--$0.39$ while DyadicAxisPE remains at $0.71$--$0.81$ (six-group comparison: Figure~\ref{fig:observability_map_main}).

These results show that the exact liftability of PE constrains the bound of observable symmetry consistently across the three PEs. Within-group second representatives and the extension to 6 symmetry groups, the null-PE limit, and the complementary PE-over-observable dominance observation are documented in Appendix~\ref{app:prefix_layers},~\ref{app:identity_pe},
and~\ref{app:grid} respectively.

\section{Conclusion}
\label{sec:conclusion}

In the tested PE-MLP setting, trained neural network weights do not expose
geometric symmetry as an absolute property of the target function, but through
an observable symmetry set determined by the positional encoding, the observable,
and training. This work formalized this dependence via the two-factor
$(\phi,\Phi)$ structure of $G_\mathrm{obs}$ and showed that the PE's exact
liftability imposes a structural upper bound on the exact observable symmetry. The
same principle explains both positive and negative cases: transformations inside
the PE's exact-liftable group tend to yield lower operational scores under the tested readouts
($D_4$-aligned responses under DyadicAxisPE, $D_3$ / $D_6$-aligned lower scores under TriAxisPE, and a $\pi$-rotation response under RFF).
However, non-liftable transformations are structurally suppressed.

These results suggest that PE selection should be viewed not only as a choice
affecting approximation quality, but also as a design parameter for post-hoc
symmetry readout. DyadicAxisPE is suitable for lattice-like symmetries up to
$D_4$, while TriAxisPE shifts the exact-liftability bound to $D_6$ by replacing
coordinate axes with three $120^{\circ}$-separated axes, thereby producing lower $D_3$ scores under the tested operational readout, which is suppressed under the DyadicAxisPE.
At the same time, exact liftability is necessary but not sufficient: training
residuals, observable choice, and feature-space redundancy also affect whether
symmetry becomes visible at the weight level.

The broader message is that not only what has been learned, but also
what is in principle observable, is constrained by the combination of
representation and readout observable. This ``limits of observability''
perspective may extend beyond symmetry detection to other attempts to infer
internal structure from the trained model weights.

\paragraph{Limitations and future directions.}
This work is limited to 2D SDF MLPs, and extending the framework to 3D data or
other architectures such as CNNs and Transformers requires identifying analogous
structural constraints. The automatic choice of optimal observables and a theory
of approximate discriminative power outside the exact-liftability regime remain
open. Future applications include PE design guided by representation-theoretic
classification, post-hoc selection of equivariant architectures, and quality
verification of the trained neural fields.

\bibliographystyle{plainnat}
\bibliography{references}

\clearpage

\appendix

\setcounter{figure}{0}
\renewcommand{\thefigure}{A\arabic{figure}}

\setcounter{table}{0}
\renewcommand{\thetable}{A\arabic{table}}

\section{Proofs of Lemma~\ref{lem:exact-lift} and Proposition~\ref{prop:hierarchy}}
\label{app:proofs}

\begin{proof}[Proof of Lemma~\ref{lem:exact-lift} (Exact liftability of structured PEs)]
We verify each claim in this section. For each PE, we construct (or show non-existence of) a linear map $\rho(g)$ such that $\phi(g\,x) = \rho(g)\,\phi(x)$.

\textbf{(i) DyadicAxisPE: $D_4$ exact lift.}
Each component of the DyadicAxisPE is independent per coordinate with frequencies $\omega_k = 2^k\pi$ ($k=0,\dots,K{-}1$).
\emph{$\pi/2$ rotation.} $R_{\pi/2}$ acts as $(x_1, x_2) \mapsto (-x_2, x_1)$. For each frequency $\omega_k$, $\sin(\omega_k (-x_2)) = -\sin(\omega_k x_2)$ and $\cos(\omega_k (-x_2)) = \cos(\omega_k x_2)$, so $R_{\pi/2}$ is represented as a permutation-and-sign matrix $P_{\pi/2}^{(k)}$ that swaps the $x_1$ and $x_2$ components with appropriate sign adjustments at each $k$. Overall, $\phi(R_{\pi/2}\,x) = P_{\pi/2}\,\phi(x)$. Because $R_{\pi/2}^2 = R_{\pi}$ and $R_{\pi/2}^4 = \mathrm{id}$ are constructed analogously, the entire rotation subgroup $\langle R_{\pi/2} \rangle$ of $D_4$ is exactly lifted.
\emph{Axis reflections.} The reflection $\sigma_{x_1}: (x_1,x_2) \mapsto (x_1, -x_2)$ acts as a diagonal sign matrix $D_{\sigma}$ flipping only $\sin(\omega_k x_2)$ (by the oddness of $\sin$ and evenness of $\cos$). The reflection $\sigma_{x_2}$ is analogous to this. Since $\langle R_{\pi/2}, \sigma_{x_1} \rangle = D_4$, the full $D_4$ group is exactly lifted.

\textbf{(i') DyadicAxisPE: non-liftability of $\pi/3$ and $2\pi/3$.}
The $\pi/3$ rotation $R_{\pi/3}$ acts as $(x_1, x_2) \mapsto (\tfrac{1}{2}x_1 - \tfrac{\sqrt{3}}{2}x_2,\, \tfrac{\sqrt{3}}{2}x_1 + \tfrac{1}{2}x_2)$. The components of $\phi(R_{\pi/3}\,x)$ include terms such as
\[
\sin\!\bigl(2^k\pi(\tfrac{1}{2}x_1 - \tfrac{\sqrt{3}}{2}x_2)\bigr) = \sin(2^k\pi \cdot \tfrac{1}{2}\, x_1)\cos(2^k\pi \cdot \tfrac{\sqrt{3}}{2}\, x_2) - \cos(2^k\pi \cdot \tfrac{1}{2}\, x_1)\sin(2^k\pi \cdot \tfrac{\sqrt{3}}{2}\, x_2).
\]
Every component of DyadicAxisPE is a univariate function of either $x_1$ or $x_2$ alone, so any linear combination takes the additively separable form $f(x_1) + g(x_2)$. The term above is a product of functions of $x_1$ and $x_2$ and hence is not additively separable; indeed, $\frac{\partial^2}{\partial x_1\,\partial x_2}\bigl[\sin(2^k\pi(\frac{1}{2}x_1 - \frac{\sqrt{3}}{2}x_2))\bigr] \neq 0$, whereas $\frac{\partial^2}{\partial x_1\,\partial x_2}[f(x_1)+g(x_2)] = 0$ identically. Therefore, no linear map $\rho(R_{\pi/3})$ satisfying $\phi(R_{\pi/3}\, x) = \rho(R_{\pi/3})\,\phi(x)$ exists. The same argument applies to $R_{2\pi/3} = R_{\pi/3}^2$.

\textbf{(ii) TriAxisPE: $D_6$ exact lift.}
The $2\pi/3$ rotation $R_{2\pi/3}$ satisfies $v_j^\top(R_{2\pi/3}\,x) = v_{j-1}^\top x$ (indices mod 3), and thus acts as a cyclic permutation of the three directional blocks. The $\pi/3$ rotation maps $v_j \mapsto -v_{j+1}$, which is realized as a cyclic permutation composed of sign flips on the sine components. The $x_1$-axis reflection maps $v_0 \mapsto v_0$ and $v_1 \leftrightarrow v_2$, again a linear action on the vertices. Since $\langle R_{\pi/3}, \sigma_{x_1} \rangle = D_6 \supset D_3$, the full $D_6$ symmetry is exactly lifted.

\textbf{(iii) RFF: $Z_2$ exact lift and non-liftability of general angles.}
Write $\theta_i := \omega_i^\top x + b_i$.
\emph{$\pi$ rotation.} For $R_{\pi}\, x = -x$, we have $\theta_i' := \omega_i^\top(-x) + b_i = -\theta_i + 2b_i$. By the addition formula,
\[
\begin{pmatrix} \sin(\theta_i') \\ \cos(\theta_i') \end{pmatrix}
= \begin{pmatrix} -\cos(2b_i) & \sin(2b_i) \\ \phantom{-}\sin(2b_i) & \cos(2b_i) \end{pmatrix}
\begin{pmatrix} \sin(\theta_i) \\ \cos(\theta_i) \end{pmatrix}
\]
holds for each $i$. Therefore, $\phi(R_{\pi}\, x) = \rho(R_{\pi})\,\phi(x)$ with the block diagonal matrix $\rho(R_{\pi}) = \mathrm{diag}(M_1, \dots, M_n)$.
\emph{Non-liftability of general angles.} For a rotation $R_\alpha$ with $\alpha \neq 0, \pi$, we have $\omega_i^\top R_\alpha\, x = (R_\alpha^\top \omega_i)^\top x$. Under the general position assumption, $R_\alpha^\top \omega_i \notin \{\pm\omega_j \mid j = 1,\dots,n\}$ for all $i$. Consequently, the components of $\phi(R_\alpha\, x)$ involve frequency directions that are absent from those of $\phi(x)$. Because plane waves $e^{i\omega^\top x}$ and $e^{i\omega'^\top x}$ with $\omega \neq \pm\omega'$ are linearly independent in $L^2(\mathbb{R}^2)$, $\phi(R_\alpha\, x)$ cannot be expressed as a linear combination of the components of $\phi(x)$.
\end{proof}

\begin{proof}[Proof of Proposition~\ref{prop:hierarchy} (Exact Observability Hierarchy)]
We show the two types of containment separately.

\textbf{$G_\mathrm{obs}^{\mathrm{exact}} \subseteq G_\mathrm{lift}^{\mathrm{exact}}$.}
By definition, the elements of $G_\mathrm{obs}^{\mathrm{exact}}$ are restricted to $G_\mathrm{lift}^{\mathrm{exact}}$. Hence, containment follows immediately.

\textbf{$G_\mathrm{obs}^{\mathrm{exact}} \subseteq G_\mathrm{true}$.}
Let $g \in G_\mathrm{obs}^{\mathrm{exact}}$. Then, $\Phi(T_g\,\theta) = \Phi(\theta)$. Since $\Phi$ is symmetry-sufficient (Definition~\ref{def:sym-sufficient}),
\[
f_\theta(g^{-1}\,x) = f_\theta(x) \quad \text{for all } x \in X.
\]
By the convergence assumption $\sup_{x \in X}|f_\theta(x) - f_*(x)| \leq \delta$, the triangle inequality gives for all $x$,
\[
|f_*(g^{-1}\,x) - f_*(x)| \leq |f_*(g^{-1}\,x) - f_\theta(g^{-1}\,x)| + |f_\theta(g^{-1}\,x) - f_\theta(x)| + |f_\theta(x) - f_*(x)| \leq 2\delta.
\]
Because $f_\theta(g^{-1}\,x) = f_\theta(x)$, the middle term vanishes. Under sufficient convergence ($\delta \to 0$), $f_*(g^{-1}\,x) = f_*(x)$ holds for all $x$, and hence $g \in G_\mathrm{true}$.

Combining both containments, $G_\mathrm{obs}^{\mathrm{exact}} \subseteq G_\mathrm{lift}^{\mathrm{exact}} \cap G_\mathrm{true}$.
\end{proof}

\section{Functional vs Structural Symmetry: Contrast with Dense Forward}
\label{app:dense_forward}

This appendix elaborates on how the symmetry readout via the observable $\Phi$ differs essentially from a naive dense-forward check $f_\theta(g^{-1}x) \approx f_\theta(x)$.

\paragraph{The space on which the group action acts.}
The essential distinction between dense forward and observables in this framework lies in \textbf{the space in which the group action is performed}. Dense forward uses the input-space action $g : x \mapsto gx$ directly, comparing $f_\theta(g^{-1}x)$ and $f_\theta(x)$ at many sample points, whereas the observable $\Phi$ is defined as the response to a structural transformation $T_g\theta$ mediated by the feature-space lift $\rho(g)$. For $g \in G_\mathrm{lift}^{\mathrm{exact}}$, the two coincide via $\phi(gx) = \rho(g)\phi(x)$; outside this set, they diverge that dense forward uses the actual $\phi(g^{-1}x)$ directly, whereas the observable handles $g$ only through its feature-space lift.

\paragraph{Functional vs structural symmetry.}
This distinction arises because PE plays a dual role in both training and observation: on the one hand, during training, PE serves as the feature map $\phi$ that determines how the MLP sees inputs; on the other hand, at observation time, it provides the structural representation $\rho$ of the group action. Because MLPs are universal approximators, sufficient training can realize function-level symmetry $f_\theta(g^{-1}x) \approx f_\theta(x)$ independently of PE liftability (\textbf{functional symmetry}). Invariance of observables via $\rho$ (\textbf{structural symmetry}), in contrast, is constrained by the PE's algebraic structure. Consequently, situations can arise in which a model is functionally symmetric yet structurally undetectable, a constraint intrinsic to structural readout that is unavoidable once $\rho$ is used, and precisely the constraint this work characterizes. The main text $G_\mathrm{obs}^{\mathrm{exact}}$ and hierarchy are described at the structural rather than functional level of symmetry. A concrete instance of this gap is presented for translational symmetry in Appendix~\ref{app:translation}.

\section{Translation Symmetry Readout}
\label{app:translation}

The main text focuses on rotational and reflective symmetries, and translational symmetry fits naturally into the same framework. This appendix (i) consolidates the exact lift of translation for all three PEs as a single lemma, (ii) reports PE-level verification via Procrustes residuals , and (iii) demonstrates translation detection using gram observables on trained weights.

\subsection{Exact Lift of Translation under the PEs}
\label{app:trans_lemma}

Lemma~\ref{lem:exact-lift} addressed rotations and reflections. Translation admits a uniform statement across all three PEs, which we formulate separately.

\begin{lemma}[Exact lift of translation]
\label{lem:translation-lift}
For each of DyadicAxisPE, TriAxisPE, and RFF, there exists a PE-dependent block-diagonal linear map $M(t)$ such that the translation $x \mapsto x + t$ is exactly lifted as
\[
\phi(x + t) = M(t)\,\phi(x).
\]
\end{lemma}

\begin{proof}
\textbf{DyadicAxisPE.} By the addition formula, for each frequency $\omega_k = 2^k\pi$ and coordinate $x_i$,
\[
\begin{pmatrix} \sin(\omega_k(x_i + t_i)) \\ \cos(\omega_k(x_i + t_i)) \end{pmatrix}
= \begin{pmatrix} \cos(\omega_k t_i) & \sin(\omega_k t_i) \\ -\sin(\omega_k t_i) & \cos(\omega_k t_i) \end{pmatrix}
\begin{pmatrix} \sin(\omega_k x_i) \\ \cos(\omega_k x_i) \end{pmatrix}.
\]
Therefore, $\phi_\mathrm{Dyadic}(x+t) = M(t)\,\phi_\mathrm{Dyadic}(x)$ with the block-diagonal matrix $M(t) = \mathrm{diag}_{k=0}^{K-1}(R_{\omega_0 t_1}, R_{\omega_0 t_2}, \dots, R_{\omega_{K-1} t_2})$, where each block is a $2 \times 2$ rotation matrix.

\textbf{TriAxisPE.} Apply the same addition formula argument to each directional projection $u_j := v_j^\top x$ ($j=0,1,2$). The shift $v_j^\top t$ in $u_j$ acts as a $2\times 2$ rotation at each $j$ and frequency $\omega_k = 2^k\pi$, so $\phi_\mathrm{tri}(x+t) = M(t)\,\phi_\mathrm{tri}(x)$ holds with a block diagonal $M(t)$.

\textbf{RFF.} With $\theta_i := \omega_i^\top x + b_i$, we have $\theta_i' := \omega_i^\top (x+t) + b_i = \theta_i + \omega_i^\top t$. The addition formula gives
\[
\begin{pmatrix} \sin(\theta_i') \\ \cos(\theta_i') \end{pmatrix}
= \begin{pmatrix} \cos(\omega_i^\top t) & \sin(\omega_i^\top t) \\ -\sin(\omega_i^\top t) & \cos(\omega_i^\top t) \end{pmatrix}
\begin{pmatrix} \sin(\theta_i) \\ \cos(\theta_i) \end{pmatrix}
\]
for each $i$. Therefore, $\phi_\mathrm{RFF}(x+t) = M(t)\,\phi_\mathrm{RFF}(x)$ with $M(t) = \mathrm{diag}(R_{\omega_1^\top t}, \dots, R_{\omega_n^\top t})$.
\end{proof}

\subsection{PE-level Procrustes residuals}
\label{app:trans_procrustes}

Complementing the Procrustes residuals (Table~\ref{tab:procrustes}), the residuals for translation $(0.5, 0)$ are DyadicAxisPE $5.4 \times 10^{-5}$, TriAxisPE $5.4 \times 10^{-6}$, and RFF $1.5 \times 10^{-7}$. All three PEs yield $r_P \leq 10^{-5}$, consistent with the exact lift established in Lemma~\ref{lem:translation-lift}: translation is exactly representable at the PE level for all three PEs.

\subsection{Detection of translational symmetry via prefix Gram $L_0$}
\label{app:trans_detect}

\begin{figure}[t]
\centering
\includegraphics[width=\linewidth]{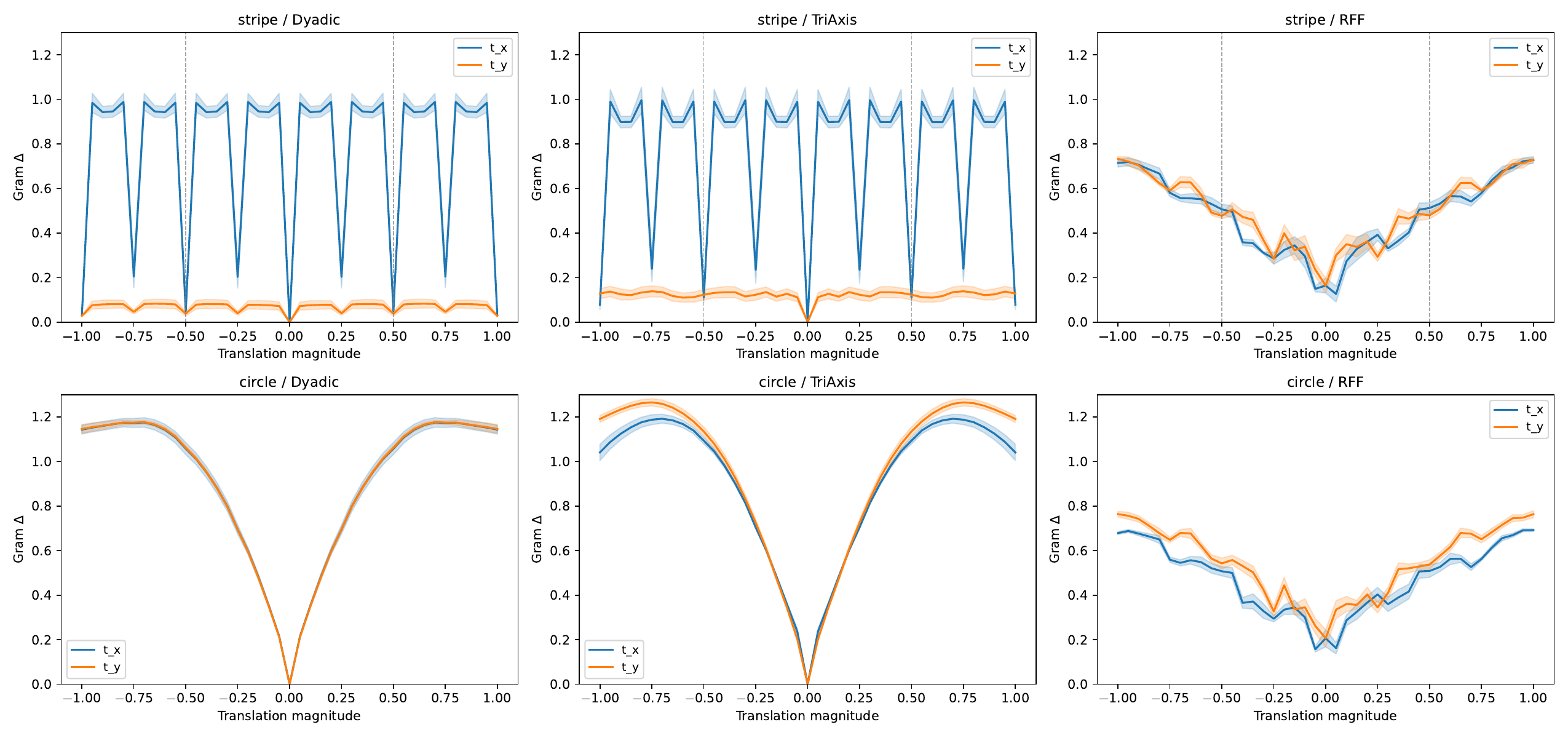}
\caption{Detection of translational symmetry: translation landscape of prefix Gram $L_0$ $\Delta$ (DyadicAxisPE: left column, TriAxis: middle column, RFF: right column). The gray dashed vertical lines indicate the true translational periods. Stripe (top row): DyadicAxisPE and TriAxis exhibit periodic sharp dips, whereas RFF shows only a broad valley at $t=0$. Circle (bottom row, translationally asymmetric): a $t=0$ valley across all three PEs. All three PEs exactly lift translation (Lemma~\ref{lem:translation-lift}), but the sufficient condition $M(t) = I$ for $G_0(T_t\theta) = M(t)^\top G_0 M(t)$ to equal $G_0$ (i.e., $\omega_i^\top t \equiv 0 \pmod{2\pi}$ simultaneously for every $\omega_i$) is met at integer multiples of the stripe period for the discrete dyadic frequencies (DyadicAxisPE/TriAxis) but fails with probability zero at $t \neq 0$ for the random frequencies (RFF); see the derivation below. A full-shape, full-layer comparison including the 2D lattice pattern (checker) is provided in Figure~\ref{fig:translation_all}.}
\label{fig:translation}
\end{figure}

Using prefix Gram $L_0$ constructed from the trained weights as the observable, Figure~\ref{fig:translation} compares the translation responses for the stripe ($p1m$) and circle (translationally asymmetric).

By Lemma~\ref{lem:translation-lift}, translation is exactly lifted for all three PEs ($\phi(x+t) = M(t)\,\phi(x)$ with block-diagonal $M(t)$ of $2\times 2$ rotation blocks), and for prefix Gram $L_0$ $G_0 = W_0^\top W_0$ we have
\[
G_0(T_t\theta) = M(t)^\top\, G_0\, M(t)
\]
exactly. A sufficient condition for $\Delta \approx 0$ (a dip) at a given $t$ is $M(t) = I$, i.e., $\omega_i^\top t \equiv 0 \pmod{2\pi}$ simultaneously for every PE frequency $\omega_i$ (each block $R_{\omega_i^\top t}$ of $M(t)$ reduces to the identity). Geometrically, each condition defines a family of parallel lines in the $t$-plane with normal $\omega_i$, $L_i = \{t : \omega_i^\top t \in 2\pi\mathbb{Z}\}$, and $M(t) = I$ requires $t \in \bigcap_i L_i$. In the 2-dimensional $t$-plane, the behavior is determined by the \textbf{number of independent directions} among $\{\omega_i\}$:
\begin{itemize}
\item \textbf{DyadicAxisPE} ($\omega_k = 2^k\pi$ along two coordinate axes) and \textbf{TriAxisPE} (the same dyadic frequencies along three axes separated by $2\pi/3$): frequencies sharing a direction are mutually dependent, so the $2K$ conditions (respectively $3K$) effectively reduce to two independent directions. Two independent constraints on 2D $t$ leave a non-trivial 2D lattice as the common solution set, which meets integer multiples of the stripe period $\Rightarrow$ \textbf{periodic dips}.
\item \textbf{RFF} ($\omega_i$ are generic-position random 2D frequencies): the $n=24$ frequencies are generically non-commensurate; more than two independent constraints overdetermine the two-dimensional translation variable, so the 24 conditions form an over-determined system on 2D $t$ (three or more independent constraints on a 2D space generically leave only $\{0\}$ as the common solution). The probability of simultaneous satisfaction at any $t \neq 0$ is zero $\Rightarrow$ \textbf{no dip outside $t=0$}.
\end{itemize}
Thus, although all three PEs exactly lift translation, periodic dips in prefix Gram $L_0$ appear only when the PE frequency set is discrete and compatible with the translational period of the shape. This is a concrete instance of the functional--structural gap discussed in Appendix~\ref{app:dense_forward}: the stripe is functionally represented by the RFF-MLP, but the structural invariance of $G_0$ under the sandwich action requires a discrete solution of $M(t) = I$, which is absent (except at $t=0$) for random-frequency PEs. For the translationally asymmetric circle, all three PEs show a valley only at $t=0$.

Note that while translation is exactly lifted at the PE representation-theory level, the estimator used in our implementation (Procrustes approximation) leaves residuals of $0.01$--$0.15$.\footnote{This residual is inherent to the Procrustes estimator; analytic estimators exploiting the block-diagonal structure of translation lifts could reduce it, which we leave to future work.}
The PE-level residuals in C.2 are \(r_P\le10^{-5}\). The larger 0.01--0.15 values refer to downstream operational Gram-score discrepancies under the estimator, not to PE liftability.
This distinction that representation-level exactness versus estimator-level approximation is important for clarifying the consistency between theory and experiments.

\subsection{Full-layer comparison}
\label{app:trans_all_layers}

The layer dependence of the translation detection is shown in Figure~\ref{fig:translation_all}.

\begin{figure}[htbp]
\centering
\includegraphics[width=\linewidth]{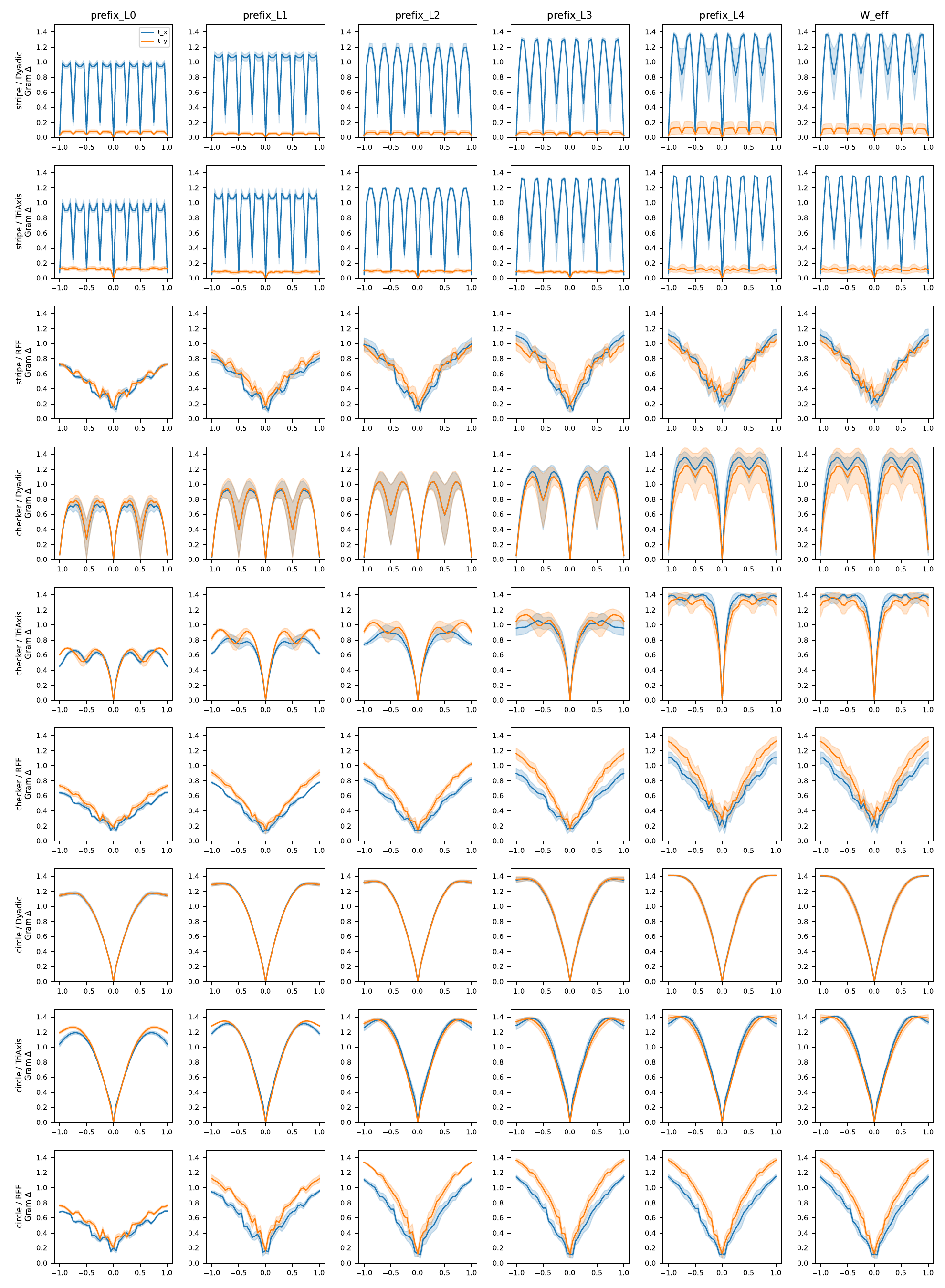}
\caption{Full-layer comparison of translation detection (3 shapes $\times$ 3 PEs $\times$ 6 observables). Rows correspond to shape / PE; columns to observables (prefix\_$L_0$--$L_4$, $W_\mathrm{eff}$).}
\label{fig:translation_all}
\end{figure}

Qualitative patterns are preserved across all observables (stripe: periodic dips under DyadicAxisPE/TriAxis, a single broad valley at $t=0$ under RFF; circle: only a $t=0$ valley under all three PEs). The dips are sharpest at $L_0$, consistent with the adoption of prefix Gram $L_0$.

\section{Layer Dependence of prefix Gram and Extension to Six Symmetry Groups}
\label{app:prefix_layers}

In the main text, we used prefix Gram $L_0$ ($G_0 = W_0^\top W_0$) as the weight-level observable for six representative shapes. This appendix (i) extends the analysis to six symmetry groups with additional shapes via a representative-angle probe at $L_0$, and (ii) supports the use of $L_0$ as the clearest readout among the tested observables by comparing all layers $L_0$--$L_4$ and $W_\mathrm{eff}$ Gram.

\paragraph{Representative-angle probe across 6 groups.}
This is a representative angle detection probe that evaluates the readout sensitivity at a single group-specific angle rather than summarizing the full group-integrated observability. The exact liftability bound patterns demonstrated with the six shapes are qualitatively maintained when extended to six symmetry groups ($O(2)$, $D_4$, $D_6$, $D_3$, $D_2$, trivial) including additional shapes (Figure~\ref{fig:observability_map_main}). Each group is probed at a single group-specific representative angle ($D_4$: $\pi/2$, $D_3$: $2\pi/3$, and so on). The structure whereby lattice groups at or below $D_4$ are stably detected with DyadicAxisPE, while $D_3$ remains non-detectable with DyadicAxisPE and RFF, is reproduced across all groups. Comparing across three PEs with TriAxisPE, DyadicAxisPE is superior for $D_4$ whereas TriAxis is superior for $D_3$, confirming the complementary pattern corresponding to each PE's exact-liftable group. However, within-group variation exists across shapes; per-shape prefix Gram results are available in Appendix~\ref{app:grid}.

\begin{figure}[htbp]
\centering
\includegraphics[width=0.9\linewidth]{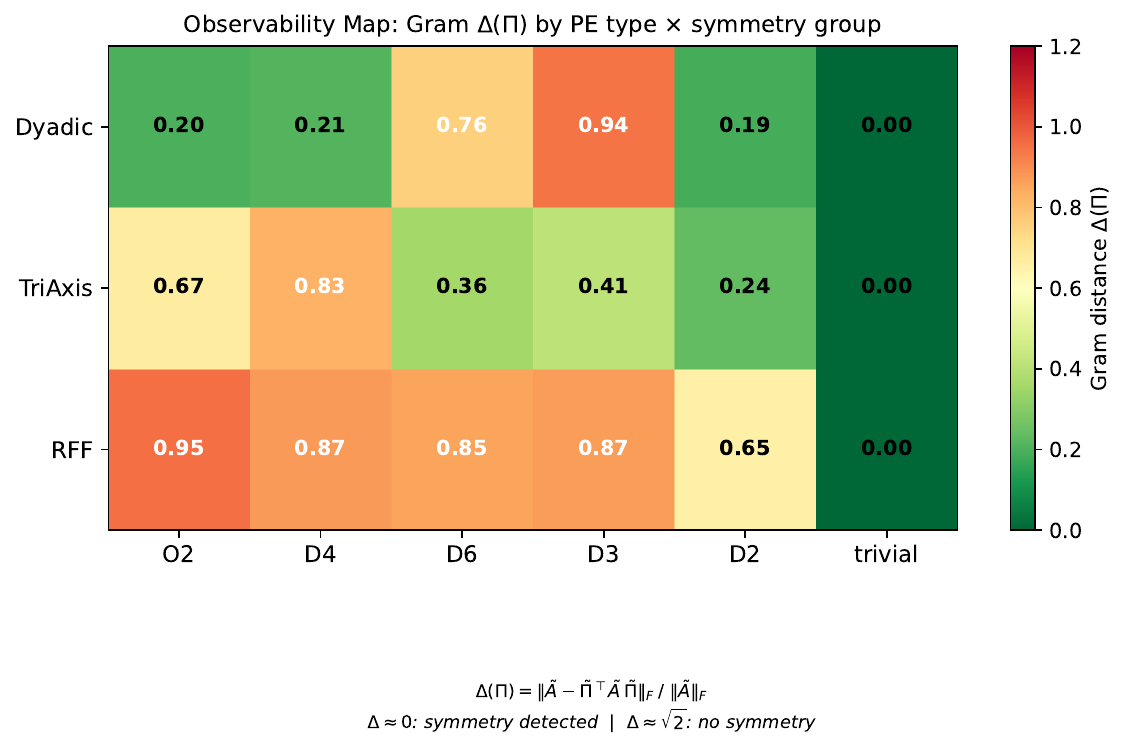}
\caption{Representative-angle detection probe (single-angle probe; not a full group-integrated score). Three PEs (DyadicAxisPE, TriAxis, RFF) $\times$ 6 symmetry groups ($O(2)$, $D_4$, $D_6$, $D_3$, $D_2$, trivial), prefix Gram $L_0$ $\Delta$ at a single group-specific angle ($D_4$: $90^{\circ}$, $D_6$: $60^{\circ}$, $D_3$: $120^{\circ}$, etc.). Each cell shows the Gram distance $\Delta(\Pi)$ averaged within each group (five seeds); $\Delta \approx 0$ indicates a low invariant-response score under the probed transform, $\Delta \approx \sqrt{2}$ indicates failure. DyadicAxisPE achieves the best score for $D_4$ whereas TriAxis achieves the best score for $D_3$/$D_6$, clearly showing the complementary pattern shaped by each PE's exact-liftable group. Within-group variation: Appendix~\ref{app:grid}. Full-layer comparison: Figure~\ref{fig:obs_map_all}.}
\label{fig:observability_map_main}
\end{figure}

\paragraph{Rotation response curves (full-layer version of Figure~\ref{fig:rotation_curves}).}
Figure~\ref{fig:rotation_curves_all} shows the rotation-angle response across all observables for 3 representative shapes. Shallower layers ($L_0$) produce sharper dips with smaller seed-to-seed variance, whereas the readout sensitivity decreases toward deeper layers. $W_\mathrm{eff}$ gram showed the lowest readout sensitivity.

\begin{figure}[htbp]
\centering
\includegraphics[width=0.8\linewidth]{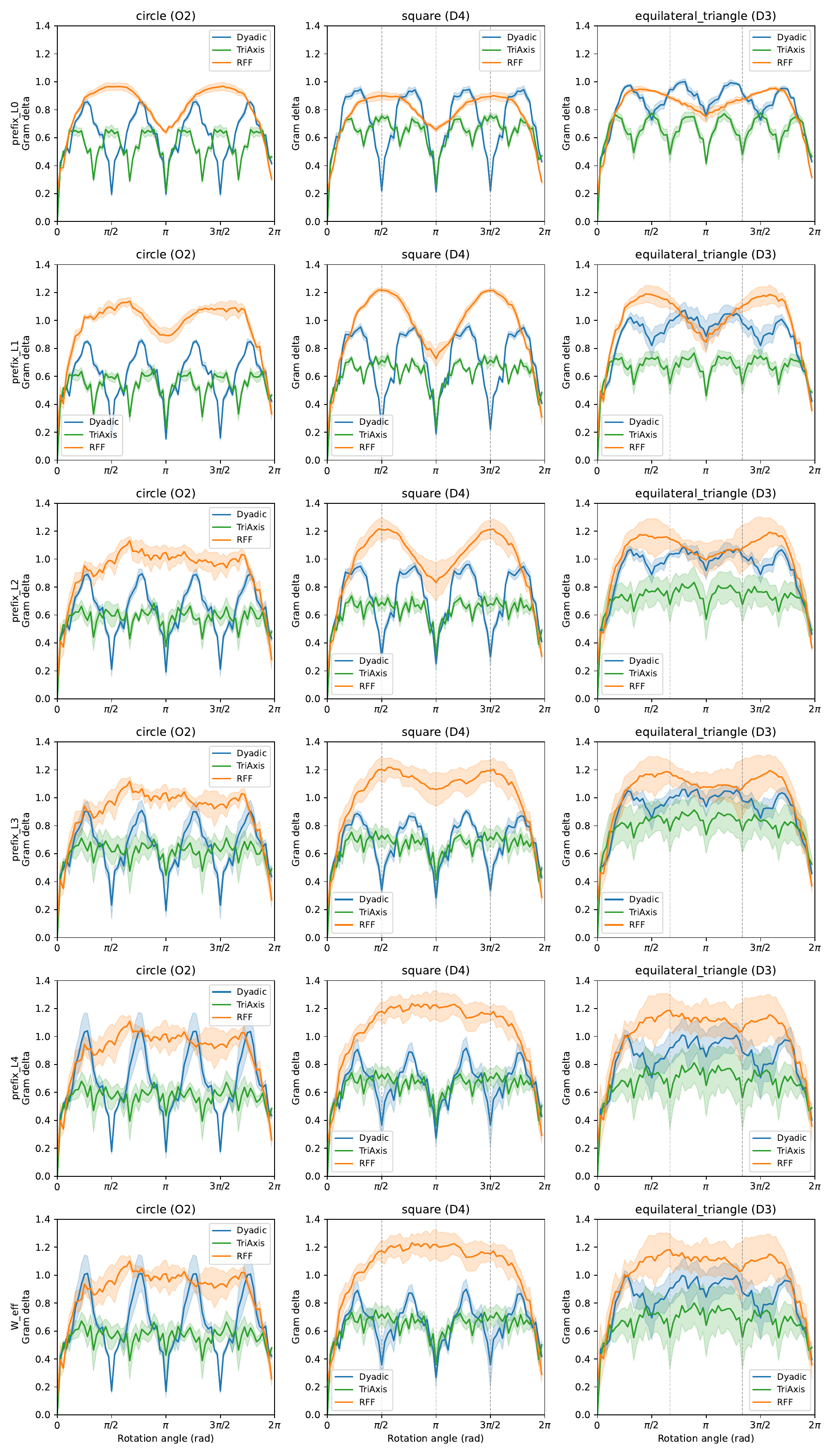}
\caption{Full-layer comparison of rotation response curves ($L_0$--$L_4$ + $W_\mathrm{eff}$). Each row corresponds to a different observable, and each column corresponds to a representative shape. Shallower layers produce sharper dips.}
\label{fig:rotation_curves_all}
\end{figure}

\paragraph{$D_3$ non-detectability (full-layer version of Figure~\ref{fig:d3_nondetect}).}
The non-detectability of $D_3$ is common across all layers, although the score levels differ by layer (Figure~\ref{fig:d3_nondetect_all}).

\begin{figure}[htbp]
\centering
\includegraphics[width=\linewidth]{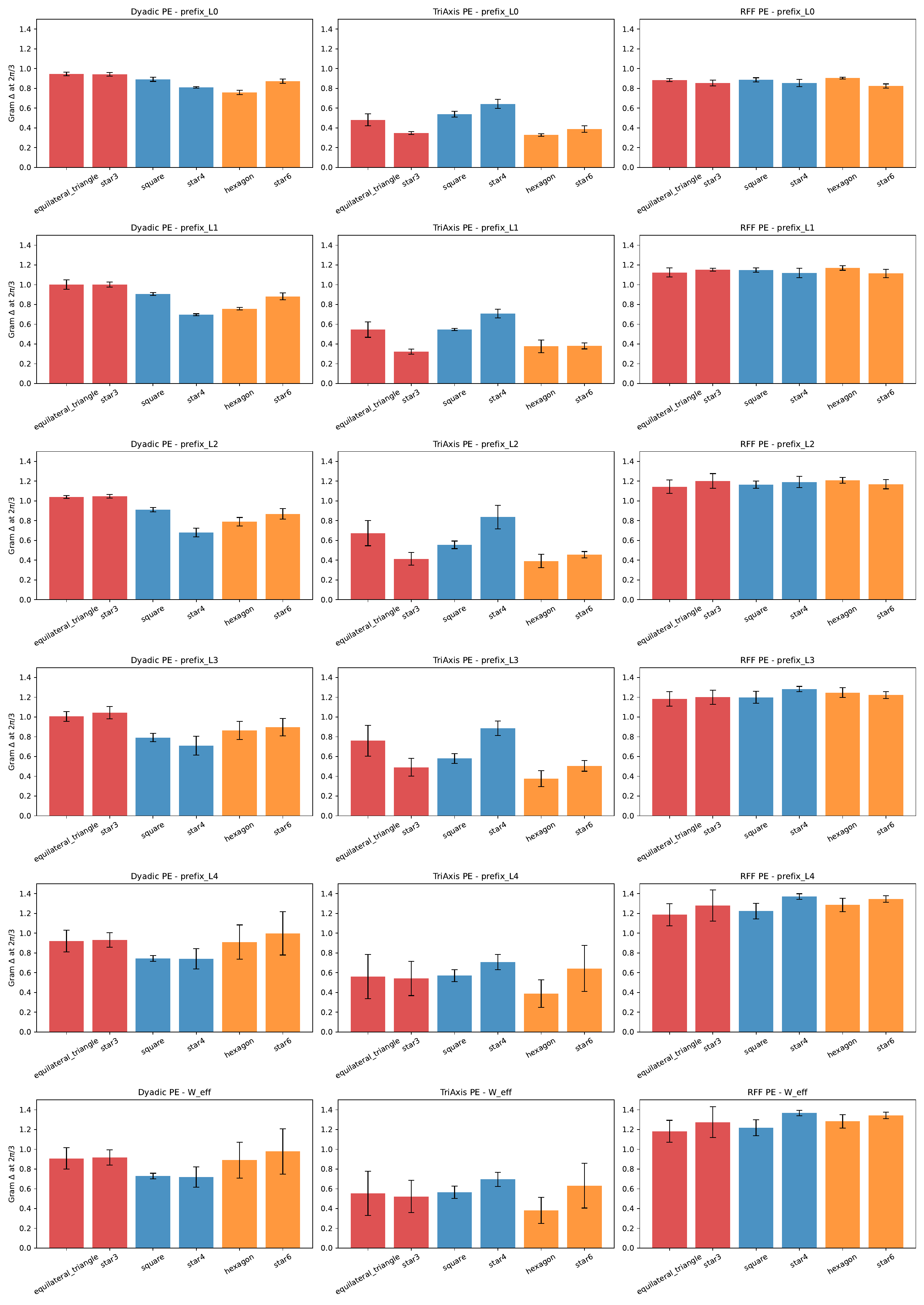}
\caption{Full-layer comparison of $D_3$ non-detectability (3 PEs: DyadicAxisPE, TriAxis, RFF).}
\label{fig:d3_nondetect_all}
\end{figure}

\paragraph{$D_3$ recovery (full-layer version at group-specific angles).}
The $D_3$ response recovery pattern via TriAxisPE (TriAxis panel of main-text Figure~\ref{fig:d3_nondetect}) is qualitatively maintained at group-specific angles across all layers (Figure~\ref{fig:d3_recovery_all}).

\begin{figure}[htbp]
\centering
\includegraphics[width=\linewidth]{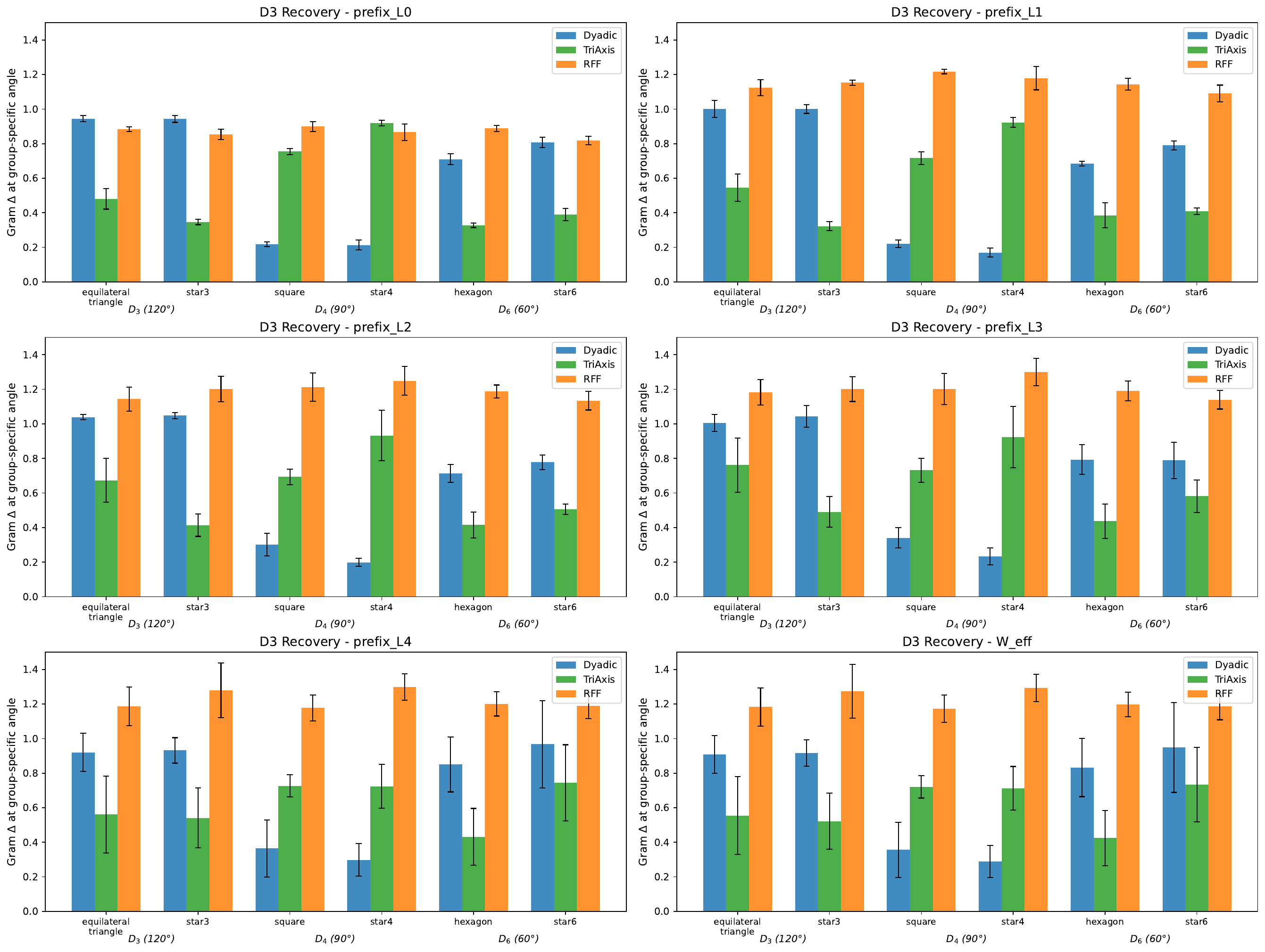}
\caption{Full-layer comparison of $D_3$ response recovery (6 observables $\times$ 3 PEs). Arranged as a 2-column $\times$ 3-row grid.}
\label{fig:d3_recovery_all}
\end{figure}

\paragraph{Observability map (full-layer version of Figure~\ref{fig:observability_map_main}).}
The observability maps for each layer are shown in Figure~\ref{fig:obs_map_all}.

\begin{figure}[htbp]
\centering
\includegraphics[width=\linewidth]{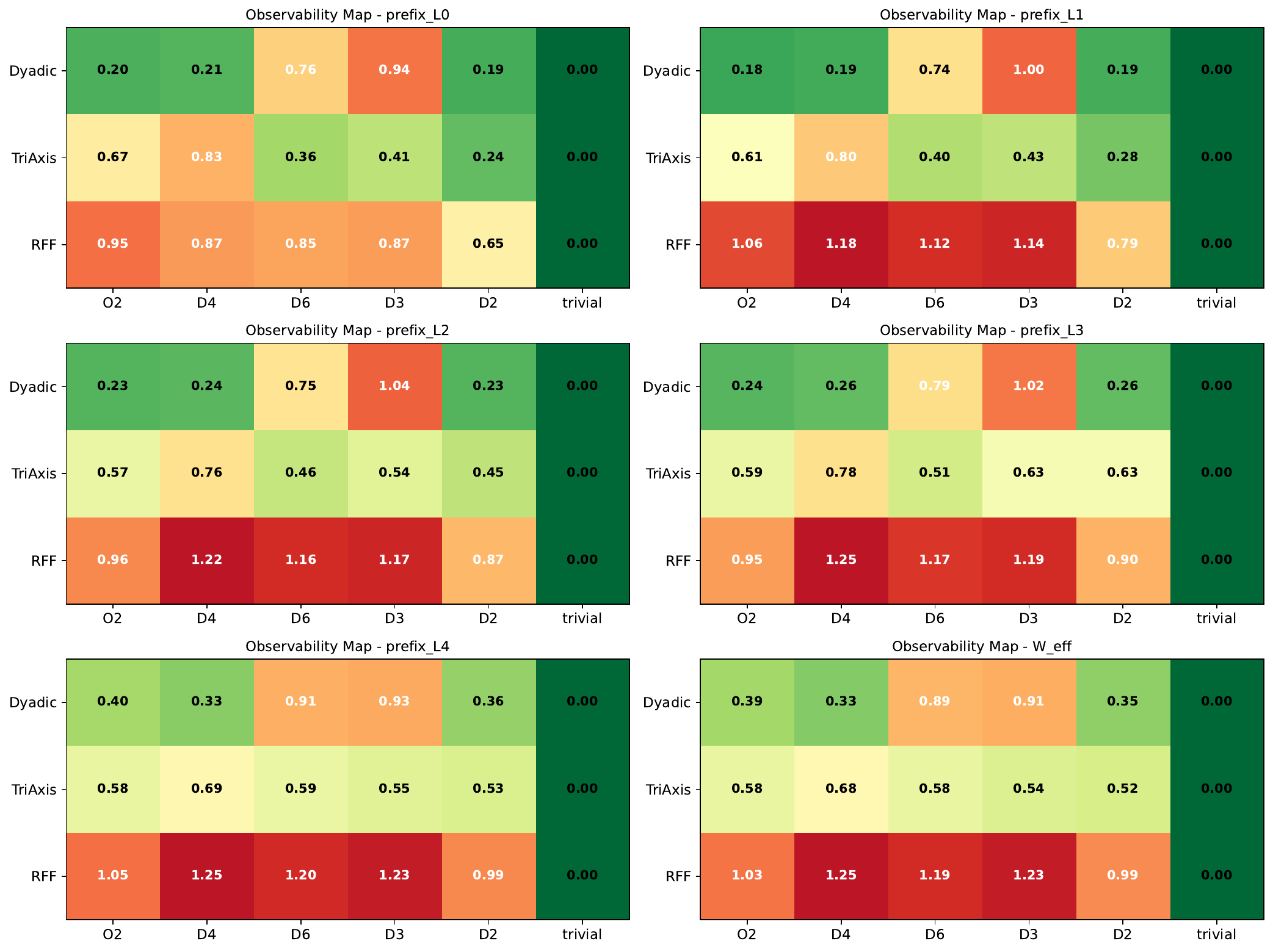}
\caption{Full-layer comparison of observability maps (3 PEs: DyadicAxisPE, TriAxis, RFF; 6 observables). Arranged as a 2-column $\times$ 3-row grid.}
\label{fig:obs_map_all}
\end{figure}

Table~\ref{tab:prefix_layer_summary} provides a quantitative summary of the results. $L_0$ achieves the best scores with the smallest seed-to-seed variance across all PEs, supporting the $L_0$ choice. The full-layer dependence of the translation detection is presented in Appendix~\ref{app:trans_all_layers} (Figure~\ref{fig:translation_all}).

\begin{table}[htbp]
\centering
\caption{Quantitative summary by observable. Mean score across all nine symmetry groups from the full 16 shapes ($O(2)$, $D_6$, $D_5$, $D_4$, $D_3$, $D_2$, $p4m$, $p1m$, trivial) in the observability map (within-group shape mean, then across-group mean; lower is better) and mean seed-to-seed standard deviation of rotation response (lower is more stable). $L_0$ is best on both metrics for all PEs.}
\label{tab:prefix_layer_summary}
\small
\begin{tabular}{l|ccc|ccc}
\hline
 & \multicolumn{3}{c|}{All-group mean score $\downarrow$} & \multicolumn{3}{c}{Rotation response seed std $\downarrow$} \\
Observable & DyadicAxisPE & TriAxis & RFF & DyadicAxisPE & TriAxis & RFF \\
\hline
Prefix $L_0$ & \textbf{0.44} & \textbf{0.54} & \textbf{0.79} & \textbf{0.019} & \textbf{0.021} & \textbf{0.024} \\
Prefix $L_1$ & 0.45 & 0.57 & 1.01 & 0.027 & 0.032 & 0.042 \\
Prefix $L_2$ & 0.50 & 0.62 & 1.07 & 0.039 & 0.052 & 0.060 \\
Prefix $L_3$ & 0.53 & 0.68 & 1.13 & 0.051 & 0.072 & 0.077 \\
Prefix $L_4$ & 0.65 & 0.67 & 1.18 & 0.102 & 0.083 & 0.101 \\
$W_\mathrm{eff}$ & 0.64 & 0.66 & 1.18 & 0.105 & 0.083 & 0.101 \\
\hline
\end{tabular}
\end{table}

\section{\texorpdfstring{Complete Results of the PE $\times$ Observable Grid}{Complete Results of the PE x Observable Grid}}
\label{app:grid}

\begin{figure}[t]
\centering
\includegraphics[width=\linewidth]{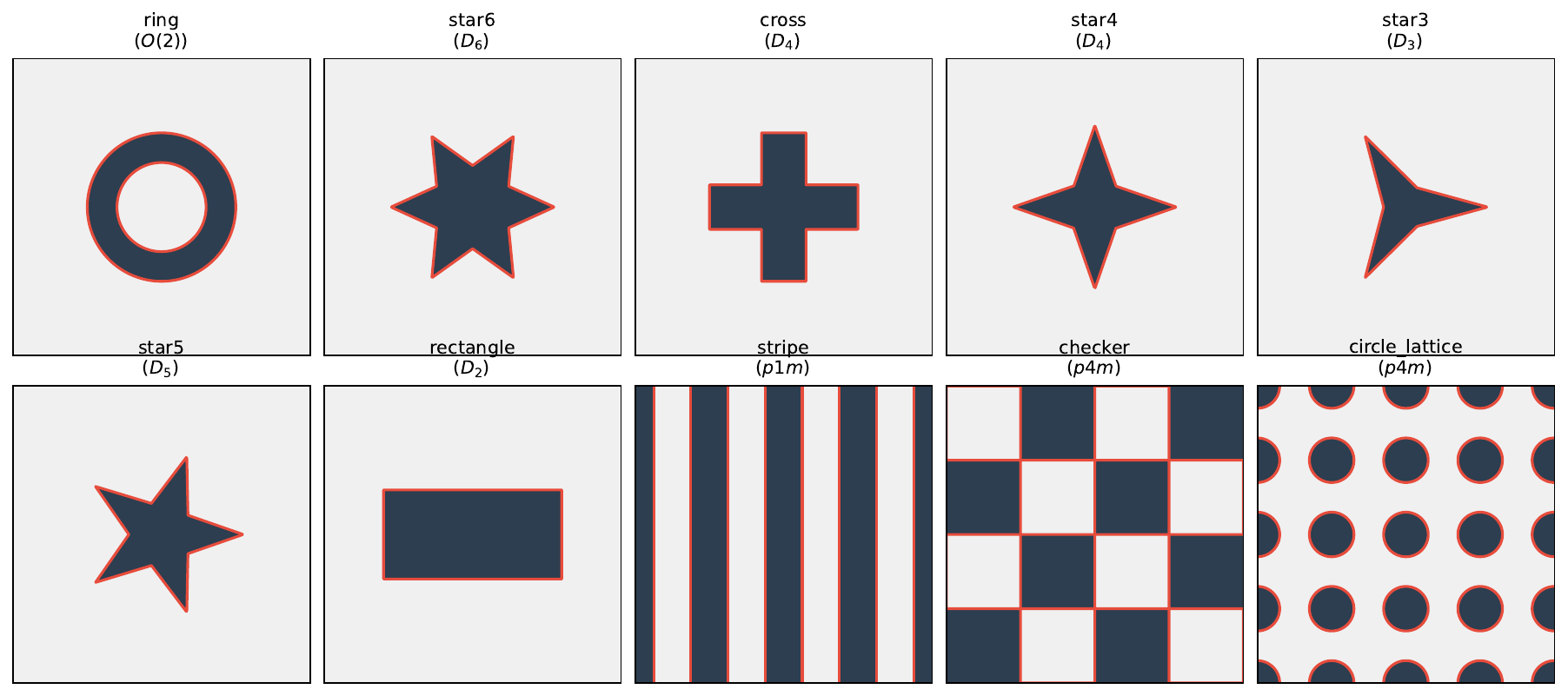}
\caption{Additional 10 shapes used in the Appendix (totaling 16 shapes with the 6 in the main text). This includes second representatives of each symmetry group (ring, star6, cross, star3, rectangle), additional groups (star4, star5), the translational pattern $p1m$ (stripe), and 2D translation patterns (checker, circle\_lattice). Stripe, checker, and circle\_lattice are used for the translation analysis in Appendix~\ref{app:translation}.}
\label{fig:shape_gallery_appendix}
\end{figure}

This appendix presents a systematic evaluation of four PEs $\times$ up to 11 observables. In particular, by comparing the effect sizes of the PE axis and the observable (gram layer/form) axis within weight-only observables, we document the complementary observation that \textbf{within weight-only readout, the PE axis dominates the observable axis}, supplementing the main text analysis. For completeness, we also include activation-level and function-level (output) observables; these lie outside the weight-only readout focus of this paper and are not used in our main claims.

\paragraph{Experimental design.}
The PE conditions were IdentityPE ($D_\mathrm{in}=2$), DyadicAxisPE K=12 (48 dimensions), TriAxis K=8 (48 dimensions), and RFF ($n=24$, 48 dimensions), totaling four types. The observables fall into three categories, totaling up to 11 types, defined as follows:

\textbf{(i) Weight Gram} (1 type): From the effective weight matrix $W_\mathrm{eff} = W_L \cdots W_0$ (product of all layer weights), we construct the extended Gram matrix $\tilde{A}$ and compute, for each of the 8 $D_4$ group elements with homogeneous orthogonal matrix $\tilde{\Pi}_g = \mathrm{diag}(\rho(g^{-1}), 1)$,
\[
\Delta(\Pi_g) = \frac{\|\tilde{A} - \tilde{\Pi}_g^\top \tilde{A}\, \tilde{\Pi}_g\|_F}{\|\tilde{A}\|_F}.
\]
The score is the mean $\Delta$ over the seven non-identity transforms ($\Delta \approx 0$: symmetry preserved; $\Delta \approx \sqrt{2}$: no symmetry). In this grid, DyadicAxisPE provides analytic $\tilde{\Pi}_g$ derived from the frequency matrix $K$, and TriAxis uses the Procrustes-estimated $\tilde{\Pi}_g$. Since $D_4$ is closed under inversion, the alternative convention $\tilde{\Pi}_g = \mathrm{diag}(\rho(g), 1)$ yields the same set of $\Delta$ values and is implementation equivalent.

\textbf{(ii) Weight-prefix Gram} ($L_0$--$L_4$, 5 types): For each layer $l$, the prefix weight product $P_l = W_l \cdots W_0$ yields the Gram matrix $P_l^\top P_l$, which is extended and scored with the same $D_4$ $\Delta$ as above. Unlike the effective-weight Gram, this preserves the layer-wise structure and is thus more likely to be symmetry-sufficient. In this grid, DyadicAxisPE admits analytic $\tilde{\Pi}_g$ and TriAxis uses Procrustes estimation; numerical estimation is possible for RFF but is not included here.

\textbf{(iii) Activation score} ($L_0$--$L_3$, 4 types): Using post-ReLU activations $h_l(x)$ at grid points $\{x_i\}$, we compute the normalized MSE over the 7 non-identity elements of $D_4$, $G_\mathrm{test} = \{R_{\pi/2}, R_\pi, R_{3\pi/2}, \sigma_{x_1}, \sigma_{x_2}, \sigma_{x_1=x_2}, \sigma_{x_1=-x_2}\}$ (with $|G_\mathrm{test}| = 7$):
\[
\mathrm{act\_score}_l = \frac{1}{|G_\mathrm{test}|} \sum_{g \in G_\mathrm{test}} \frac{\|h_l(x) - h_l(gx)\|^2}{\mathrm{Var}(h_l)}.
\]
Lower values indicate greater activation space invariance under $g$. Computable for all PE conditions.

\textbf{(iv) Output score} (1 type): The normalized MSE $\|f_\theta(x) - f_\theta(gx)\|^2 / \mathrm{Var}(f_\theta)$ of the network output $f_\theta(x)$, computed analogously to activation score. This directly measures the function-level symmetry and serves as a PE-independent upper-bound indicator. Computable for all PE conditions.

For IdentityPE and RFF, weight Gram and weight-prefix Gram are not included in this grid: IdentityPE is treated separately as a low-dimensional null condition, whereas RFF lacks a closed-form finite transform action for generic rotations and is therefore handled through operational Procrustes analyses elsewhere. TriAxis uses Procrustes-estimated PE action matrices for these observable variables. For IdentityPE and RFF, the grid produces only the activation and output scores. Their weight-side behavior is characterized separately in the IdentityPE-null analysis (see \S3.5 for the exact/operational regime distinction). Among the 16 shapes, 6 shapes (circle, ring, square, cross, star4, and circle\_lattice; henceforth \textbf{$D_4$-compatible 6 shapes})  carry all 8 elements of the origin-centered $D_4$ as true symmetries, and we pool median scores for each grid cell over these 6 shapes $\times$ 5 seeds. Because the $D_4$ transforms test $D_4$ symmetry, low scores for these shapes directly indicate stronger $D_4$-compatible readout responses.

\paragraph{Main results.}
Figure~\ref{fig:pe_observable_grid} shows the median detection scores across the PE $\times$ observable grid. Within the weight-only rows (weight Gram, weight-prefix Gram $L_0$--$L_4$), the spread across observables at a fixed PE is limited ($\sim 0.07$: DyadicAxisPE $0.16$--$0.26$, TriAxis $0.57$--$0.64$), whereas at a fixed observable, the PE-to-PE gap is large (e.g., prefix $L_0$: DyadicAxisPE $0.19$ vs TriAxis $0.61$, a $\sim 0.42$ difference), so the PE axis dominates the observable axis within the weight-only readout. In the activation rows, scores compress to $0.008$--$0.098$ across all PEs, and the PE-to-PE gap shrinks, reflecting a regime change from weight-only to activation-based readout, distinct from the within-weight-only PE dominance. Weight Gram and weight-prefix gram are restricted to DyadicAxisPE (analytic) and TriAxis (Procrustes) because analytic PE action matrices are unavailable for IdentityPE and RFF (gray cells). The overall ordering across observable categories is $W_\mathrm{eff}$ Gram ($\sim 0.25$) $>$ weight-prefix Gram ($0.16$--$0.26$) $>$ activation (act $L_3$ at $\sim 0.008$) $>$ output ($\sim 10^{-3}$--$10^{-6}$).

\paragraph{Comparison across PEs.}
DyadicAxisPE K=12 is the best in the weight-prefix gram category ($L_1$: $0.158$), substantially improving over $W_\mathrm{eff}$ gram ($0.253$). TriAxis K=8 has Procrustes-estimated weight-prefix Gram available (prefix $L_0$ median $0.610$, weight Gram $0.592$); its higher scores relative to DyadicAxisPE are consistent with the absence of $D_4$ exact lift. RFF lacks analytic weight-prefix gram in this grid, but achieves an activation score ($L_3$: $0.024$) comparable to DyadicAxisPE. IdentityPE has no weight-side observables in this grid (gray cells in the figure), and even act $L_3$ shows a score of $0.098$, approximately $8\times$ higher than the other PEs.

\subsection{IdentityPE: a null condition with insufficient redundancy}
\label{app:identity_pe}
IdentityPE ($\phi(x) = x$; output 2 dimensions, equivalent to no PE) served as a null condition with no feature expansion. Exact liftability formally holds since $\phi(gx) = g\,\phi(x)$, but the projection to two dimensions lacks the redundancy needed to encode the geometric structure in the weight space. Under the IdentityPE condition, symmetry group separation from weights was not reliably achieved with the tested weight-level observables (silhouette score: IdentityPE $=-0.133$ vs. DyadicAxisPE $=+0.285$, TriAxis $=+0.248$, RFF $=+0.027$). Meanwhile, output-level symmetry remains high (invariance ranges from $0.59$ to $1.00$ depending on the shape). The problem is not a failure of training, but rather that \textbf{without feature expansion by PE, reliable post-hoc recovery is difficult, at least with the weight-level observables used in this paper}. Unlike $D_3$ under DyadicAxisPE, whose limitation stems from the absence of an exact lift, IdentityPE's limitation is considered to arise from insufficient PE dimensionality ($D_\mathrm{in} = 2$) to provide sufficient redundancy for a reliable weight-level readout, despite the formal existence of an exact lift. In other words, exact liftability is necessary but insufficient for post-hoc detection. This is clearly an exploratory interpretation; disentangling the interaction with the estimator and observable choice, as well as additional controlled experiments (e.g., learned linear lift to higher dimensions), remains a task for future work.

\paragraph{TriAxisPE representative-angle probe results.}
The numerical details behind the TriAxisPE representative-angle probe reported in Figure~\ref{fig:observability_map_main} (Appendix~\ref{app:prefix_layers}) are as follows (prefix Gram $L_0$ $\Delta$, mean over five seeds). For $D_3$ shapes, TriAxisPE substantially outperformed DyadicAxisPE (star3: DyadicAxisPE $0.94$ $\to$ TriAxis $0.35$; equilateral\_triangle: $0.94$ $\to$ $0.48$). TriAxisPE is similarly superior for $D_6$ shapes (hexagon: DyadicAxisPE $0.71$ vs. TriAxis $0.33$; star6: DyadicAxisPE $0.81$ vs. TriAxis $0.39$). Conversely, for $D_4$ shapes, DyadicAxisPE is superior (square at $\pi/2$: DyadicAxisPE $0.22$ vs. TriAxis $0.75$). All structured PEs shared the same 48-dimensional representation ($4 \times 12 = 6 \times 8 = 2 \times 24 = 48$), so the score differences reflected the encoding structure rather than dimensionality.

\paragraph{Cost--detection trade-off.}
Weight-prefix Gram ($L_1$--$L_3$) achieves median scores of $0.158$--$0.203$ (DyadicAxisPE K=12) without activations, substantially improving over $W_\mathrm{eff}$ Gram ($0.253$) and offering a favorable trade-off between computational cost and readout sensitivity.

\begin{figure}[t]
\centering
\includegraphics[width=\linewidth]{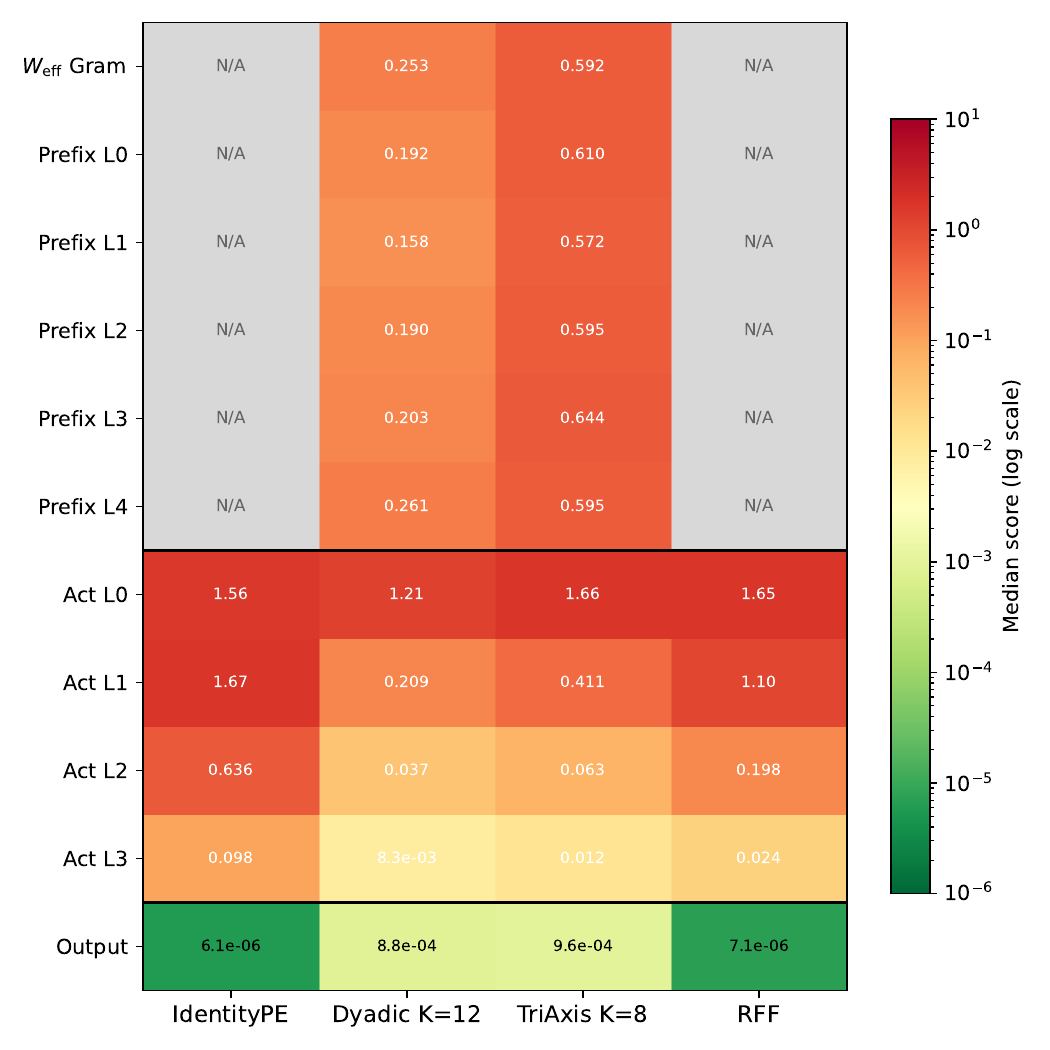}
\caption{Median scores over the 6 $D_4$-compatible shapes (circle, ring, square, cross, star4, circle\_lattice) across the PE $\times$ observable grid (4 PEs: IdentityPE, DyadicAxisPE, TriAxis, RFF; log scale; lower is better). Within the weight-only rows (prefix $L_0$--$L_4$, $W_\mathrm{eff}$ gram), the PE-to-PE gap is large, while the observable-to-observable spread within a fixed PE is small; switching to the activation rows compresses scores across all PEs. Gray cells indicate PE-observable combinations for which PE action matrices are unavailable. DyadicAxisPE uses analytic PE action matrices, whereas TriAxis uses Procrustes estimation for weight Gram and weight-prefix Gram. IdentityPE and RFF weight-side observables are gray cells.}
\label{fig:pe_observable_grid}
\end{figure}

\section{\texorpdfstring{Learning Dynamics of $D_\mathrm{op}$ and Threshold Sensitivity}{Learning Dynamics of D_{op} and Threshold Sensitivity}}
\label{app:dynamics}

\paragraph{Absence of false-positive symmetry detection and formation by training.}
A false positive occurs when $D_\mathrm{op}\not\subseteq G_\mathrm{true}$; we therefore check whether $D_\mathrm{op}\subseteq G_\mathrm{true}$ holds in the tested grid. $D_\mathrm{op} \not\subseteq G_\mathrm{true}$) is a prerequisite for the framework's utility; therefore, we first verify this empirically. Under DyadicAxisPE / Weight Gram with all eight elements of $D_4$ as the transform family, $D_\mathrm{op}$ was evaluated across the 6 main-text shapes $\times$ 3 seeds $\times$ 10 checkpoints (epoch 0--2000), yielding 0 false-positive detections in all 180 conditions (Figure~\ref{fig:trained_vs_random}). No false positives were observed, even at random initialization (epoch 0). Training progression selectively lowered the scores for transforms in $G_\mathrm{true}$, while the $D_3$ shape (equilateral\_triangle) maintained high scores across all epochs, confirming that the structural non-detectability is a stable pattern independent of training. Prefix Gram $L_0$ also showed 0 false positives under the same conditions. This verification is limited to these specific conditions and is not a theoretical guarantee of general PE--observable combinations.

\begin{figure}[t]
\centering
\includegraphics[width=\linewidth]{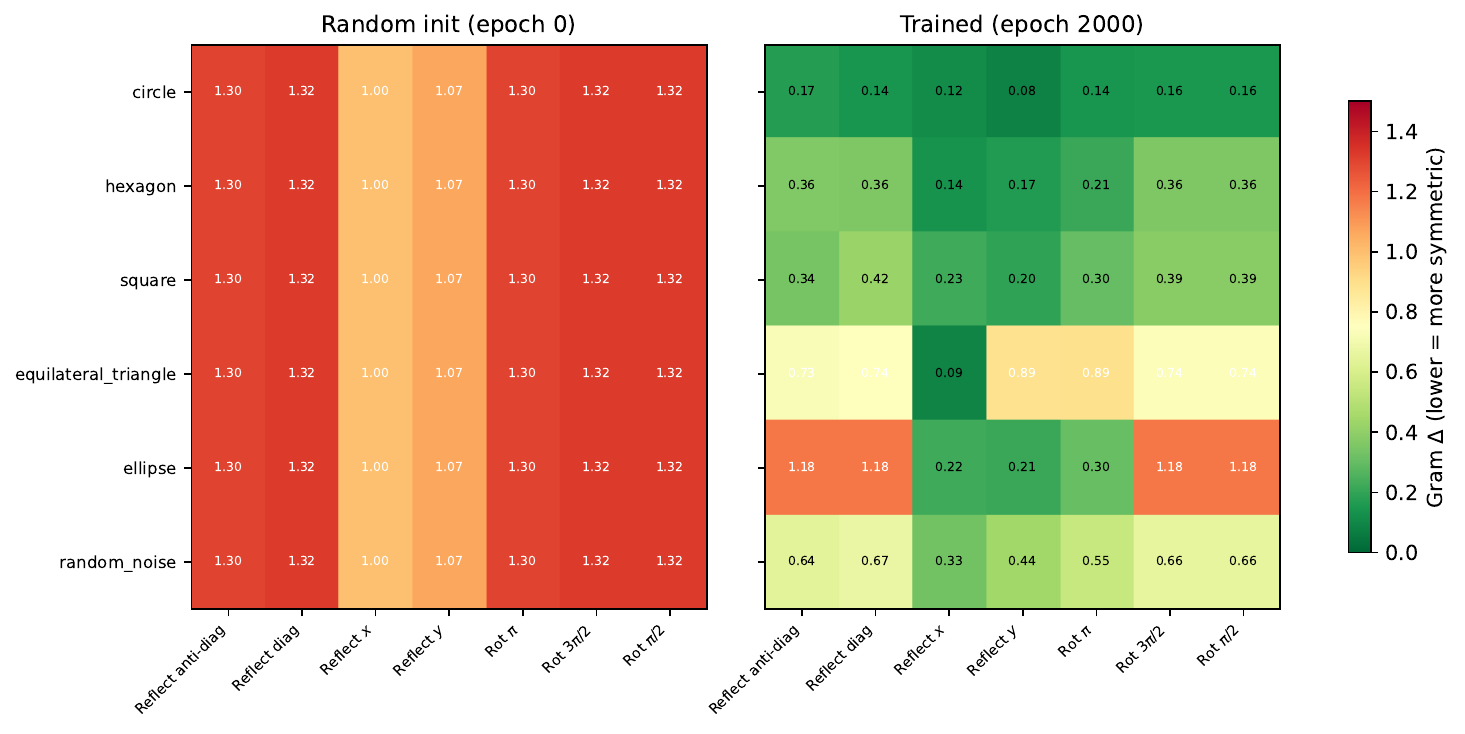}
\caption{Comparison of symmetry scores between random initialization (epoch 0) and after training (epoch 2000) (DyadicAxisPE K=12, Weight Gram, $D_4$ transforms). At initialization, scores are high for all transformations (no symmetry detected), and after training, the reductions are larger for transformations in $G_\mathrm{true}$, with the low-score detections concentrated on the true-symmetry side. This is consistent with the detected symmetries being formed by training rather than PE bias.}
\label{fig:trained_vs_random}
\end{figure}

\paragraph{Rationale for $\varepsilon = 0.05$ and threshold sensitivity.}
The Gram distance $\Delta$ is a normalized Frobenius distance $\|A - \Pi^\top A \Pi\|_F / \|A\|_F$; it equals 0 for exact symmetry and is $O(1)$ for unrelated transformations. Being a ratio, $\varepsilon = 0.05$ corresponds to a 5\% relative Frobenius error, which is a natural threshold for normalized scores. In our $D_4$ transform experiments (DyadicAxisPE K=12 / Weight Gram, epoch 2000), non-symmetry scores are all $\geq 0.24$, so $\varepsilon = 0.05$ provides a wide margin against false positives. Table~\ref{tab:threshold_comparison} compares detection counts across $\varepsilon = 0.02$--$0.10$: false positives remain 0 for $\varepsilon\le0.05$, while false positives first appear at $\varepsilon\ge0.07$. Because different observables produce raw scores at different scales, caution is required when applying a common threshold; however, the Gram distance $\Delta$ used in this work is a normalized score, and $\varepsilon = 0.05$ provides a robust operating point.

\begin{table}[t]
\centering
\caption{Detection results at different thresholds $\varepsilon$ (DyadicAxisPE K=12 / Weight Gram / $D_4$ transforms, epoch 2000, 6 shapes $\times$ 3 seeds).}
\label{tab:threshold_comparison}
\begin{tabular}{ccccc}
\hline
$\varepsilon$ & Detected & TP & FP & Conclusions stable \\
\hline
0.02 & 18 & 18 & 0 & yes \\
0.03 & 18 & 18 & 0 & yes \\
0.05 & 18 & 18 & 0 & yes \\
0.07 & 22 & 21 & 1 & no \\
0.10 & 27 & 25 & 2 & no \\
\hline
\end{tabular}
\end{table}

\section{Numerical Procrustes Estimation for RFF}
\label{app:rff_scaling}

For the RFF, analytic PE action matrices are unavailable; therefore, $\hat{\rho}(g)$ is constructed via numerical Procrustes estimation \cite{schonemann1966generalized}. The procedure is as follows: (i) sample input points $\{x_i\}_{i=1}^N$ uniformly from $[-1,1]^2$, (ii) compute transformed inputs $\{gx_i\}$, (iii) obtain PE outputs $X = [\phi(x_i)]^\top \in \mathbb{R}^{N \times D}$ and $Y = [\phi(gx_i)]^\top \in \mathbb{R}^{N \times D}$, (iv) compute the SVD of $M = Y^\top X$ as $M = U\Sigma V^\top$ and set $\hat{\rho}(g) = UV^\top \in O(D)$. The Procrustes residual $r_P(g) = \|Y - X\hat{\rho}(g)^\top\|_F / \|Y\|_F$ quantifies the estimation quality (Table~\ref{tab:procrustes}). In our experiments, $N = 512$ was used as default.

\end{document}